%% file: neurips_2026.tex
\documentclass{article}

\PassOptionsToPackage{numbers, sort, compress}{natbib}
\usepackage[preprint]{neurips_2026}

\usepackage[utf8]{inputenc} 
\usepackage[T1]{fontenc}    
\usepackage{hyperref}       
\usepackage{url}            
\usepackage{booktabs}       
\usepackage{amsfonts}       
\usepackage{nicefrac}       
\usepackage{microtype}      
\usepackage{xcolor}         

\usepackage{amsmath}    
\usepackage{amssymb}    
\usepackage{bm}         
\usepackage{graphicx}   
\usepackage{amsthm}
\usepackage{amsfonts}
\usepackage{algorithm}
\usepackage{algorithmic}
\usepackage{natbib}
\usepackage{tikz}
\usetikzlibrary{arrows.meta,positioning,calc}

\usepackage{enumitem}
\usepackage{adjustbox}
\usepackage{placeins}

\usepackage[capitalize,noabbrev]{cleveref}
\usepackage{subcaption}

\theoremstyle{plain}
\newtheorem{theorem}{Theorem}[section]

\newtheorem{lemma}[theorem]{Lemma}

\theoremstyle{definition}
\newtheorem{definition}[theorem]{Definition}

\theoremstyle{remark}

\definecolor{eccpurple}{rgb}{255, 0, 255}

\newcommand\ecc[1]{}
\newcommand\ma[1]{}
\newcommand\az[1]{}

\newcommand{\LLM}{LLM}
\newcommand{\LLMs}{LLMs}

\title{Open Problems in Constitutional Preference Reconstruction}

\makeatletter
\let\orig@bottomtitlebar\@bottomtitlebar
\renewcommand\@bottomtitlebar{%
  \orig@bottomtitlebar
  \vspace{-0.5em}%
}
\renewcommand\@notice{}%
\makeatother

\newcommand*\samethanks[1][\value{footnote}]{\footnotemark[#1]}

\author{%
	Eleanor Clifford\thanks{Co-first authors. Correspondence to
\href{mailto:eleanor.clifford@cl.cam.ac.uk}{eleanor.clifford@cl.cam.ac.uk} and
\href{mailto:ma2151@cst.cam.ac.uk}{ma2151@cst.cam.ac.uk}.}\\
	Imperial College London\\
	University of Cambridge
	\And
	Michael Amir\samethanks[1]\\
	University of Cambridge
	\And
	Arduin Findeis\\
	University of Cambridge
	\And
	Aaron Zhao\\
	Imperial College London
	\And
	Robert Mullins\\
	University of Cambridge
}

\begin{document}


\maketitle
\vspace*{-0.5em}

\input{sections/abstract.tex}
\input{sections/introduction.tex}
\input{sections/related.tex}
\input{sections/motivation.tex}
\input{sections/method.tex}
\input{sections/results.tex}
\input{sections/discussion.tex}

\FloatBarrier
\input{sections/conclusion.tex}
\input{sections/acknowledgements.tex}
\FloatBarrier

\bibliography{bibliography}
\bibliographystyle{unsrtnat}


\newpage
\appendix

\input{appendices/code.tex}
\input{appendices/hyperparameters.tex}

\newpage
\input{appendices/principle_refinement_ablation.tex}
\input{appendices/constitution_size_ablation.tex}
\input{appendices/lemma_proof.tex}

\newpage
\input{appendices/compute.tex}
\input{appendices/example_constitutions_across_methods.tex}

\newpage
\input{appendices/greedy_vs_naive.tex}

\newpage
\input{appendices/example_constitutions_across_llms.tex}

\newpage
\input{appendices/cross_llm_agreement.tex}

\newpage
\input{appendices/full_results.tex}

\newpage
\input{appendices/prompt_templates.tex}


\end{document}

%% file: sections/abstract.tex
\begin{abstract}

Pairwise preference data is widely used for training and evaluating language models (e.g., RLHF), but each datapoint records a \emph{choice}, not the rationale behind it. Methods such as Inverse Constitutional AI (ICAI) attempt to improve interpretability by compressing datasets into short ``constitutions'' of natural-language principles. We argue this framing is under-specified: a flat list of principles is not yet an executable decision rule because it leaves principle composition implicit. We use the pairwise setting as a testbed to empirically characterize three open problems in constitutional methods. First, principle quality is hard to measure: coverage and accuracy are useful but incomplete proxies for end-to-end reconstruction. Second, \emph{composition is ambiguous}: holding principles fixed, different executors (LLM judge versus majority vote) agree only $73\%$ of the time. Third, \emph{constitutions differ between LLMs}: cross-model vote agreement is $73\%$, whereas intra-model agreement is $81\%$. Across PRISM, AlpacaEval, and Chatbot Arena, we show that principle refinement (ICAI+) may be a first step towards ameliorating these problems: inter-executor agreement rises to $78\%$, and transparent executors match LLM judge accuracy ($66\%$ vs.\ $67\%$). Our results highlight that constitutions should be evaluated as \emph{constitution--executor systems}, with implications for LLMs-as-a-judge broadly.

\end{abstract}

%% file: sections/introduction.tex
\section{Introduction}

Pairwise preference data, consisting of judgments that one option is better than another, is an important interface between humans and modern machine learning systems. This type of data is used in methods such as reward modeling and reinforcement learning from human feedback (RLHF)~\citep{christiano2017deep, ouyang2022training} and direct preference optimization (DPO)~\citep{rafailov2023direct}, and forms the basis of notable evaluation benchmarks for large language models (\LLMs), such as Chatbot Arena \cite{chiang2024chatbotarenaopenplatform} and AlpacaEval \cite{alpaca_eval}. Yet preference data is difficult to audit, because it records \emph{decisions}, not the \emph{rationales} behind them: the reasons for a choice are typically implicit, contingent on context, and driven by interacting (sometimes conflicting)  considerations rather than a single stable rule~\citep{tversky1991loss,casper2023open}. This opacity is undesirable in any setting where models are trained or evaluated based on preference data, but it is especially problematic when preference data is treated as a proxy for human values.

A natural response is to \emph{interpret} preference datasets by translating them into something legible. Inspired by Constitutional AI~\citep{bai2022constitutional}, recent work proposes to compress a preference dataset into a ``constitution'': a list $\mathcal{C}=\{P_1,\dots,P_d\}$ of natural-language principles intended to explain and reconstruct the dataset~\citep{findeis2025inverse}. A constitution is attractive because it turns a large set of opaque choices into something legible. But this artifact is easy to overread. A list of principles may name some of the criteria behind the data, but it still does not say how those criteria should be composed into decisions.

This paper argues that the ``constitution'' framing is incomplete for a structural reason:
\textbf{a constitution is not yet an executable decision rule}.
A flat list of principles does not specify how conflicts are resolved, how context affects which principles apply, or how to handle underspecified trade-offs (e.g., correctness vs.\ brevity, harmlessness vs.\ helpfulness).
In practice, such missing details often dominate LLM-as-a-judge results: annotations can vary substantially with the judge model and prompting even when the written criteria are held fixed~\citep{zheng2023judgingllmasajudgemtbenchchatbot,liu2023g_eval,chiang2024chatbotarenaopenplatform}.

\input{appendices/the_pipeline.tex}

To make the missing composition step explicit, we separate constitutional preference reconstruction into three roles: a \emph{discoverer}, an \emph{annotator}, and an \emph{executor}. \Cref{fig:pipeline} illustrates this setup. Given labeled preference pairs $\mathcal{D}=\{(A_i,B_i,y_i)\}_{i=1}^N$, where $y_i\in\{\text{A},\text{B}\}$ is the preferred response, a discoverer $M_D$ proposes candidate principles $\widetilde{\mathcal{P}}=\{P_j\}_{j=1}^m$. An annotator $M_A$ then applies each principle to each pair, producing votes $V_{ij}\in\{\text{A},\text{B},\text{N/A}\}$, where N/A means that the principle is not applicable (including tie scenarios). A subset of these principles forms a constitution $\mathcal{C}\subset\widetilde{\mathcal{P}}$, and an executor $M_E$ turns the constitution into final reconstructed labels $\hat y_i$ on held-out comparisons.

The executor is the step that is often hidden. In ICAI-style pipelines, it is usually an \LLM{} prompted with the constitution. That \LLM{} does not transparently ``apply'' the constitution; it decides which principles matter, how conflicts should be resolved, and what to do when the principles are vague or incomplete. Thus, performance ``under a constitution'' is really performance of a \emph{constitution--executor system}, not the constitution alone. The same issue appears in criteria-based \LLM{} evaluation more broadly, where the effective rule depends on both the written rubric and the judge model~\citep{chiang2024chatbotarenaopenplatform,MTBench_Bai_2024,liu2023g_eval,kim2024prometheus,zheng2023judgingllmasajudgemtbenchchatbot}.

This framework raises \textbf{three practical open problems} which we empirically characterize:\vspace{-0.4em}

\begin{enumerate}
    \item \textbf{Principle quality is hard to measure.}
    Per-principle metrics such as coverage and accuracy are weak proxies for end-to-end reconstruction.
    Different discovery procedures can yield constitutions that reconstruct equally well while differing substantially in semantic content.

    \item \textbf{Principle composition is ambiguous.}
    Holding a constitution fixed, different plausible executors (LLM-as-a-judge prompting versus transparent aggregation rules such as majority, priority, or tree-based execution) can disagree materially.
    This executor sensitivity quantifies the extent to which the written constitution under-specifies the effective decision rule.

    \item \textbf{Constitutions are not reliably transferable between \LLMs.}
    Different \LLMs{} can discover different principles from the same data, and can also interpret the same principles differently when used as annotators.
    This limits explainability: if different LLMs disagree on the meaning of a principle, the principle text itself cannot fully describe the decision process.\vspace{-0.2em}
\end{enumerate}

\paragraph{Approach and findings}
Across PRISM~\citep{kirk2024prism}, AlpacaEval~\citep{alpaca_eval}, and Chatbot Arena~\citep{chiang2024chatbotarenaopenplatform}, we quantify executor sensitivity via inter-executor agreement and measure cross-\LLM{} portability via agreement on principle applicability and vote direction. We show that naive ICAI discovery often produces principles that are either narrow (near-zero coverage) or broad but semantically underspecified, effectively deferring key decisions to the executor.
As a constructive step, we evaluate \textbf{ICAI+}, a lightweight refinement procedure that (i) filters out low-quality candidates on a small subsample and (ii) rewrites marginal principles using targeted counterexamples.
Empirically, ICAI+ improves coverage and reduces some signs of underspecification, narrowing the gap between opaque LLM execution and transparent executors: with refined principles, simple aggregation rules and small tree-based executors perform comparably to LLM-as-a-judge. 

\paragraph{Contributions}

We make four contributions.
\begin{enumerate}
	\item \textbf{Problem formalisation}. We formalize constitutional preference reconstruction as a \emph{discoverer--annotator--executor} stack, and argue that constitutions should be evaluated as constitution--executor systems.

	\item \textbf{We show that executor choice matters.} Under naive ICAI principles, LLM execution agrees with majority vote only $73.4\% \pm 4.2$ of the time on test, while majority vote underperforms LLM execution ($60.7\% \pm 1.9$ vs.\ $64.5\% \pm 2.0$; \cref{tab:agreement_pair,tab:reconstruction}).

	\item \textbf{We show that principles have limited cross-LLM transferability.} Vote agreement for naive principles is $73.0\% \pm 2.1$, well below an $80.5\% \pm 3.2$ intra-model baseline. (\cref{tab:cross_llm_agreement_icai_pair}).

	\item \textbf{We introduce ICAI+ as a refinement case study.} It raises LLM--majority agreement from $73.4\% \pm 4.2$ to $78.4\% \pm 4.1$, and makes majority vote competitive with LLM-as-a-judge execution ($64.6\% \pm 2.6$ vs.\ $65.5\% \pm 2.1$; \cref{tab:agreement_pair,tab:reconstruction}). These gains reduce, but do not eliminate, the open problems above.
\end{enumerate}

Rather than solving constitutional preference reconstruction, the focus of this paper is to reveal several unaddressed challenges associated with it: principle quality is hard to evaluate, composition is under-specified, and natural-language principles are model-dependent. ICAI+ is useful because it shows that improving principles can reduce some of these problems, but all three remain open.

%% file: appendices/the_pipeline.tex

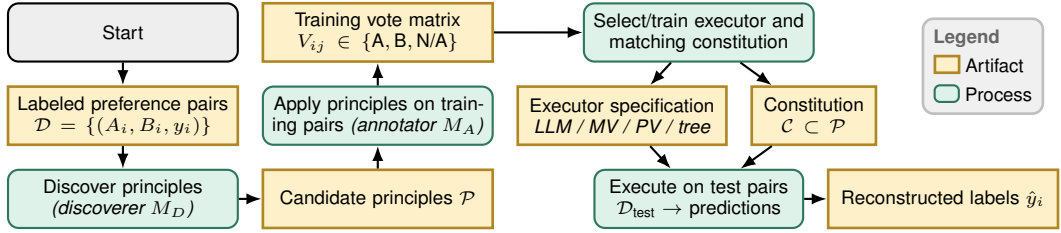
\begin{figure}[t]
\centering

\definecolor{fyellow}{HTML}{af8725}
\definecolor{flyellow}{HTML}{fef0c9}
\definecolor{fgreen2}{HTML}{2e765f}
\definecolor{flgreen2}{HTML}{dff2e9}
\definecolor{grey}{HTML}{eeeeee}
\newcommand{\localfontstyle}{\scriptsize\sffamily}

\begin{adjustbox}{max width=\linewidth}
\begin{tikzpicture}[
  proc/.style={
    draw=fgreen2,
    line width=0.8pt,
    text=black,
    rounded corners,
    fill=flgreen2,
    align=center,
    inner sep=3.5pt,
    font=\localfontstyle,
    text width=0.2\columnwidth,
    minimum height=2\baselineskip
  },
  art/.style={
    draw=fyellow,
    line width=1pt,
    text=black,
    fill=flyellow,
    align=center,
    inner sep=3.5pt,
    font=\localfontstyle,
    text width=0.2\columnwidth,
    minimum height=2\baselineskip
  },
  halfart/.style={art, text width=0.15\columnwidth},
  arrow/.style={-{Latex[length=2mm]}, thick},
  note/.style={font=\localfontstyle, align=center},
  node distance=3mm
]


\node[proc, draw=black, fill=grey] (start) {Start};
\node[art, below=of start] (data) {Labeled preference pairs $\mathcal{D}=\{(A_i,B_i,y_i)\}$};
\node[proc, below=of data] (discover) {Discover principles \emph{(discoverer $M_D$)}};
\node[art, right=of discover] (pool) {Candidate principles $\mathcal{P}$};

\node[proc, above=of pool] (annot) {Apply principles on training pairs \emph{(annotator $M_A$)}};
\node[art, above=of annot] (votes) {Training vote matrix $V_{ij}\in\{\text{A},\text{B},\text{N/A}\}$};

\node[art, right=of annot, text width=0.18\columnwidth] (E) {Executor specification\\ \emph{LLM / MV / PV / tree}};
\node[art, right=of E, text width=0.1\columnwidth] (C) {Constitution\\$\mathcal{C}\subset\mathcal{P}$};

\node[proc, anchor=south, yshift=3mm] (train)  at ($(E.north)!0.4!(C.north)$) {Select/train executor and matching constitution};
\node[proc, anchor=north, yshift=-3mm, text width=0.18\columnwidth] (apply) at ($(E.south)!0.4!(C.south)$) {Execute on test pairs $\mathcal{D}_{\text{test}}$ $\rightarrow$ predictions};

\node[art, right=of apply] (pred) {Reconstructed labels $\hat y_i$};

\draw[arrow] (start) -- (data);
\draw[arrow] (data) -- (discover);
\draw[arrow] (discover) -- (pool);
\draw[arrow] (pool) -- (annot);
\draw[arrow] (annot) -- (votes);
\draw[arrow] (votes) -- (train);
\draw[arrow] (train) -- (C);
\draw[arrow] (train) -- (E);
\draw[arrow] (C) -- (apply);
\draw[arrow] (E) -- (apply);
\draw[arrow] (apply) -- (pred);

\definecolor{flgray}{HTML}{f0f0f0}
\node[draw=black!25, line width=1pt, rounded corners, fill=flgray,
      inner sep=4pt, font=\localfontstyle, align=left, anchor=north east,
      xshift=-2mm, yshift=-2mm
     ] at (train.north-|pred.east) {
        \textcolor{black!60}{\textbf{Legend}}\\[3pt]
        \tikz[baseline=0.1ex]\draw[draw=fyellow,line width=1pt,fill=flyellow,rounded corners=0pt]
            (0,0) rectangle (3.2mm,2.2mm);~Artifact\\[3pt]
            \tikz[baseline=0.1ex]\draw[draw=fgreen2,line width=0.8pt,fill=flgreen2,rounded corners=2pt]
            (0,0) rectangle (3.2mm,2.2mm);~Process
     };

\end{tikzpicture}
\end{adjustbox}
\caption{\textbf{Preference reconstruction as a discoverer--annotator--executor stack.} Each datapoint is $(A_i,B_i,y_i)$, where $A_i$ and $B_i$ are the two candidate responses and $y_i\in\{\text{A},\text{B}\}$ denotes the preferred response. The annotator produces votes $V_{ij}\in\{\text{A},\text{B},\text{N/A}\}$ for principle $P_j$ on pair $i$ (N/A = not applicable). Constitution selection and executor fitting occur using training votes and may depend on the executor class (e.g., majority-vote-optimized selection for flat constitutions).}
\vspace{-0.8em}
\label{fig:pipeline}
\end{figure}

%% file: sections/related.tex
\section{Background}

\paragraph{Human Preference Data}
Preference data captures choices between options but not the underlying rationale. In AI alignment, methods like RLHF~\citep{christiano2017deep, ouyang2022training} and DPO~\citep{rafailov2023direct} use this data to shape model behavior, but they encode preferences implicitly within model weights, resulting in opaque policies that are difficult to audit~\citep{casper2023open}. The central challenge is transforming large volumes of opaque preference data into an understandable model of the underlying decision-making process.

\paragraph{Constitutional AI and ICAI}
Constitutional AI (CAI)~\citep{bai2022constitutional} uses manually defined principles to guide model training. Follow-up work by \citet{huang2024collective} extends Constitutional AI to develop the constitution collaboratively with a broader public. Inverse Constitutional AI (ICAI)~\citep{findeis2025inverse} reverses this: given a preference dataset, it automatically discovers principles that explain the observed judgments. ICAI samples pairs from the dataset, prompts an LLM to propose principles justifying the observed preference, and then applies these principles to reconstruct labels by again prompting an LLM. Follow-up work by \citet{movva2025WhatsMyHuman} explores the use of sparse auto-encoders (SAEs) as an alternative to ICAI's LLM-based pipeline, and work by \citet{kingslin2026democratic} investigates improving ICAI principles via structured persona debate. While ICAI and related methods provide valuable insights, we argue they leave key aspects of the decision procedure under-specified.

\paragraph{LLM-as-a-Judge}
Using LLMs to evaluate text quality has become widespread~\citep{zheng2023judgingllmasajudgemtbenchchatbot}. Methods range from direct prompting to criteria-based evaluation (G-Eval~\citep{liu2023g_eval}) to fine-tuned evaluator models (Prometheus~\citep{kim2024prometheus}), to the use of prompt-specific rubrics~\citep{gunjal2025rubrics}. A common assumption is that providing explicit criteria makes evaluation transparent and reproducible. We show this assumption is questionable: the same criteria can yield different judgments depending on the executor model.

\paragraph{Interpretable Machine Learning}
Our work draws on interpretable ML, which emphasizes models that are inherently understandable~\citep{rudin2019stop}. Decision trees and forests~\citep{breiman1984classification,breiman2001random,ke2017lightgbm} provide explicit, auditable decision rules. We compare these as explicit executors against implicit LLM-based composition.

%% file: sections/motivation.tex
\newpage
\section{Preference Reconstruction: Open Problems}

We formalize three open problems that arise when a constitution is treated as an ``explanation'' of a preference dataset. Throughout, we distinguish three roles in the reconstruction stack: a \emph{discoverer} model $M_D$ that proposes principles, an \emph{annotator} model $M_A$ that operationalizes principles into per-example votes, and an \emph{executor} that composes votes into final predictions (\Cref{fig:pipeline}).

\subsection{Problem 1: Measuring Principle Quality}

A natural goal is to summarize the quality of a principle $P$ by how often it applies and how often it agrees with ground truth when it applies.

\begin{definition}[Coverage and accuracy]
Let $\mathcal{D}=\{(A_i,B_i,y_i)\}_{i=1}^{|\mathcal{D}|}$ be a dataset of labeled pairwise comparisons (e.g., a train or test split), where $y_i\in\{\text{A},\text{B}\}$ denotes which response is preferred.
Given a principle $P$ and an annotator-produced vote $v_i^P\in\{\text{A},\text{B},\text{N/A}\}$ for each comparison $i$ (with N/A meaning ``not applicable''), \textbf{coverage} is the fraction of comparisons with $v_i^P\neq \text{N/A}$, and \textbf{accuracy} is the fraction of covered comparisons on which $v_i^P=y_i$.
\label{def:coverage_and_accuracy}
\end{definition}




\subsection{Problem 2: Ambiguity in Principle Composition}

A constitution is a list of principles; applying it requires \emph{composing} principle-level votes into a single decision. Many ICAI-style pipelines delegate composition to an LLM judge prompted with the constitution, which implicitly supplies precedence relations, context-dependence, and trade-offs.

We treat executor sensitivity as a diagnostic. Holding principles fixed, we compare:
(i) \emph{implicit} LLM execution (LLM-as-a-judge with the constitution text), and
(ii) \emph{transparent} executors over principle votes: majority vote, priority vote (highest-accuracy applicable principle), and tree-based models trained on vote features.
We report reconstruction accuracy (Definition \ref{def:reconstruction_accuracy}) and \textbf{inter-executor agreement} (pairwise agreement rates on predicted labels). Disagreement does not imply that any executor is ``wrong''; rather, it quantifies underspecification in the written constitution and the extent to which composition logic is being supplied by the executor.

\begin{definition}[Reconstruction accuracy and inter-executor agreement]
Fix a split $\mathcal{S}=\{(A_i,B_i,y_i)\}_{i=1}^n$ and an executor $E$ that outputs predictions $\hat y_i^{(E)}\in\{\text{A},\text{B}\}$. For two executors $E$ and $E'$, the \textbf{reconstruction accuracy} of $E$ on $\mathcal{S}$ is defined as
\begin{align}
\mathrm{RA}_{\mathcal{S}}(E) &:= \frac{1}{n}\sum_{i=1}^n \mathbf{1}\!\left[\hat y_i^{(E)} = y_i\right],
\end{align}
and their \textbf{(inter-executor) agreement} on $\mathcal{S}$ as
\begin{align}
\mathrm{Agr}_{\mathcal{S}}(E,E') &:= \frac{1}{n}\sum_{i=1}^n \mathbf{1}\!\left[\hat y_i^{(E)} = \hat y_i^{(E')}\right].
\end{align}
\label{def:reconstruction_accuracy}
\end{definition}

\subsection{Problem 3: Cross-\LLM{} Transferability}

Constitutions are written in natural language and therefore appear model-agnostic, but in practice both discovery and annotation can be model-dependent.
First, a discoverer model may induce one set of principles from a dataset while another model induces a semantically different set. Second, even when the principle text is fixed, different annotator models may disagree about whether the principle applies, or about which response better satisfies it.
Quantitatively, we focus on annotator transferability. For fixed principle text, we measure agreement between annotator models on two questions: (i) whether the principle applies to a comparison (coverage agreement) and (ii) conditional on both models marking the principle applicable, which response the principle favors (vote agreement). We also examine qualitatively whether different discoverer models can induce semantically different constitutions from the same dataset, but still achieve similar reconstruction accuracy (Appendix~\ref{sec:llm-constitution-comparison}).

%% file: sections/method.tex
\section{Methods: ICAI and ICAI+}

Baseline ICAI~\citep{findeis2025inverse} seeks a constitution $\mathcal{C}=\{P_1,\dots,P_d\}$: a set of natural-language principles intended to explain observed preferences in a dataset $\mathcal{D}=\{(A_i,B_i,y_i)\}_{i=1}^N$. A discoverer \LLM{} $M_D$ is prompted with sampled comparisons to propose principles, then uses an annotator \LLM{} $M_A$ to produce per-example votes for each principle. A constitution is then selected and used to reconstruct labels, often by prompting an executor \LLM{} $M_E$ with the full constitution and the comparison.

This pipeline exposes two interpretability limitations that motivate our study.
First, \textbf{flat principle lists under-specify composition}: a list of principles does not uniquely determine precedence, trade-offs, or context-based gating.
Second, \textbf{naive discovery often yields low-utility principles}: many candidates have extremely low coverage, while broader candidates tend to be semantically vague, effectively deferring decision-making back to the executor.

We compare naive ICAI discovery to ICAI+, a refinement procedure that filters out low quality principles early on a small subsample of the training set and rewrites marginal principles using examples where they succeed and fail. We ask: \emph{do improvements in per-principle coverage and accuracy reliably translate into improvements in end-to-end reconstruction?} Our results suggest that coverage and accuracy alone are insufficient proxies once a baseline level of signal is achieved because redundancy, ambiguity, and cross-principle interactions also matter.

Unless stated otherwise, we report per-principle coverage and accuracy on the held-out test split $\mathcal{D}_{\text{test}}$ using the annotator $M_A$ associated with the reconstruction stack under study.
For ICAI+, filtering and rewriting (Algorithm~\ref{alg:icaiplus}) uses a disjoint calibration subset $\mathcal{S}\subset\mathcal{D}_{\text{train}}$ and does not access $\mathcal{D}_{\text{test}}$.

\subsection{ICAI+: Early Filtering and Targeted Refinement}

We study a lightweight refinement procedure (ICAI+) that modifies only the \emph{discovery} step, aiming to produce principles that are (i) higher-coverage and (ii) less semantically ambiguous.
ICAI+ performs (a) early filtering of low-signal candidates and (b) targeted rewriting of marginal candidates using small counterexample sets.
This is described in Algorithm~\ref{alg:icaiplus}, with further explanation of hyperparameters in Appendix~\ref{sec:icaiplus-parameters}.
Hyperparameters for experiments run can be seen in Appendix~\ref{sec:hyperparameters}.

\begin{figure}[b]
\vspace{-2.2em}
\begin{algorithm}[H]
\caption{ICAI+ principle refinement (high-level)}
\label{alg:icaiplus}
\begin{algorithmic}[1]
	\STATE \textbf{Input:}
		dataset $\mathcal{D}$;
		discoverer $M_D$;
		annotator $M_A$;
		min.\ annotations $k_{\textrm{min}}$;
		max.\ annotations $k_{\textrm{max}}$;
		min.\ votes $r_{\textrm{min}}$;
		min.\ coverage $c_{\textrm{min}}$;
		min.\ accuracy $a_{\textrm{min}}$;
		refinement threshold $a_r$;
		target count $d$.

\STATE Initialise $\mathcal{P}$ to the empty set. Sample a labeled calibration subset $\mathcal{S}\subset\mathcal{D}$ of size $k_{\textrm{max}}$.
\WHILE{$|\mathcal{P}| < d$}
	\STATE Generate an pool of candidate principles $\widetilde{\mathcal{P}}$ using ICAI-style prompts on sampled comparisons.
	\FOR{$(P, S_k)\in\widetilde{\mathcal{P}} \times \mathcal{S}$}
		\STATE Apply $P$ on $S_k$ with annotator $M_A$ to obtain vote $v_k^P$.
		\STATE Over $v^P$, compute number of applicable votes $r^P$, coverage $c^P$, and accuracy $a^P$.
		\STATE Discard $P$ if ($k > k_{\textrm{min}}$ and $c^P < c_{\textrm{min}}$) or ($r^P > r_{\textrm{min}}$ and $a^P < a_{\textrm{min}}$).
		\STATE If $a_{\textrm{min}} < a^P < a_r$, collect a small set of examples where $P$ succeeds/fails and prompt $M_D$ to rewrite $P$ to improve generality and reduce ambiguity. Add this to $\widetilde{\mathcal{P}}$.
	\ENDFOR
	\STATE Add remaining principles from $\widetilde{\mathcal{P}}$ to $\mathcal{P}$.
\ENDWHILE
\STATE Return the final principle set $\mathcal{P}$.
\end{algorithmic}
\end{algorithm}
\vspace{-2em}
\end{figure}

\subsection{Votes and Executors}

Given a constitution $\mathcal{C}$, the annotator $M_A$ produces votes $v_{ij}\in\{\text{A},\text{B},\text{N/A}\}$ for each comparison $i$ and principle $P_j$.
We compare \emph{opaque execution} to \emph{transparent execution}:
\begin{itemize}
    \item \textbf{Opaque (LLM-as-a-judge):} prompt an executor \LLM{} $M_E$ with the constitution text and comparison $(A_i,B_i)$ to directly predict $\hat y_i$.
    \item \textbf{Transparent:} Either \textit{rule-based}: majority vote over $\{v_{ij}\}$; and priority vote (highest-accuracy applicable principle), or \textit{learned}: a decision tree (scikit-learn~\citep{pedregosa2011scikit}) or LightGBM gradient-boosted trees~\citep{ke2017lightgbm}  trained on vote features.
\end{itemize}

\section{Experimental Setup and Datasets}

We evaluate on three \LLM{} preference datasets: AlpacaEval~\citep{alpaca_eval} (instruction-following preferences), PRISM (conversational preferences with social and ethical nuance) \cite{kirk2024prism}, and Chatbot Arena~\citep{chiang2024chatbotarenaopenplatform} (crowdsourced real-world user preferences). Dataset sizes can be seen in Appendix~\ref{sec:hyperparameters}. To reduce position effects, we symmetrize the data by swapping $A$ and $B$ and flipping votes.

For principle discovery, annotation, and implicit LLM execution we use GPT-4o and GPT-4o-mini for fair comparison against ICAI~\citep{findeis2025inverse}, along with Gemini 2.5 Flash, DeepSeek v3.1 Chat, GPT 5.4-nano, and GPT-5.4-mini. \citep{openai2024gpt4ocard,deepseek2024deepseekv3,gemini25pushingfrontier,openaigpt5card}. For explicit tree-based executors we use scikit-learn (decision trees) and LightGBM (gradient-boosted trees) \cite{ke2017lightgbm}. Prompt templates are provided in Appendix~\ref{app:prompts}.


For each dataset and model configuration, we: (1) discover candidate principles using naive ICAI or ICAI+; (2) annotate the training set with principle-level votes $v_{ij}\in\{\text{A},\text{B},\text{N/A}\}$; (3) define an executor using training annotations (e.g., select a subset of principles for flat constitutions, or fit a tree-based executor); (4) execute on held-out test data to predict preferences; and (5) evaluate reconstruction accuracy and inter-executor agreement.

To define the flat executors, we greedily add to the constitution the principles that maximize majority-vote reconstruction accuracy on training data (Algorithm~\ref{alg:greedy_constitution}). This encourages selecting \emph{complementary} principles that correct one another's errors, rather than near-duplicates with overlapping coverage; details and comparison with other principle selection strategies appear in Appendix~\ref{sec:gmv-metric-comparison}. We chose 30 as the number of principles to add based on an ablation study which can be seen in Appendix~\ref{sec:constitution-size-ablation}.

To measure annotator transferability, we sample 10 ICAI and 10 ICAI+ principles from each discoverer model and have all models annotate these; we report cross-model agreement in Section~\ref{section:results.transferability}. Unless stated otherwise, we report means over the $18$ dataset--model configurations ($3$ datasets $\times$ $6$ model stacks) $\pm$ 95\% confidence $t$-intervals across these $18$ values.


%% file: sections/results.tex
\vspace{-0.2em}
\section{Results}
\label{section:results}
\vspace{-0.2em}

We report two complementary views of preference reconstruction:
(i) \textbf{end-to-end reconstruction accuracy} (agreement with held-out labels), and
(ii) \textbf{executor sensitivity} measured by \textbf{inter-executor agreement} when holding the constitution fixed.
Table~\ref{tab:reconstruction} summarizes reconstruction accuracy across executors for naive ICAI versus ICAI+ principles.

\vspace{-0.1em}
\subsection{Problem 1: Principle Quality}
\vspace{-0.1em}

ICAI+ improves principle-level utility. In a representative case (AlpacaEval, DeepSeek v3.1 Chat), ICAI+ shifts the discovered principles toward higher coverage and moderately higher accuracy (Figure~\ref{fig:coverage_accuracy}). Qualitatively, naive ICAI frequently proposes narrow surface-form principles, while ICAI+ produces more task-relevant criteria (Figure~\ref{fig:example_constitutions}). Clustering visualizations further suggest reduced redundancy under ICAI+ (Figure~\ref{fig:kmeans}).

\input{figures/principle_clustering.tex}


However, these per-principle improvements translate into only modest end-to-end gains: averaged across datasets and model settings, reconstruction increases by \(\approx 1\) percentage points, from $64.5\% \pm 2.0$ (ICAI, LLM-as-a-judge) to $65.5\% \pm 2.1$ (ICAI+, LLM-as-a-judge) (Table~\ref{tab:reconstruction}), suggesting that coverage/accuracy alone are poor proxies once principles clear a baseline level of signal. More broadly, \emph{principle quality is not one-dimensional}: a principle can be accurate but redundant, broad but under-specified, or effective only under a particular composition rule. As a result, reconstruction depends on interactions among principles and on the executor, not just on any single per-principle metric. Despite improved coverage and accuracy, defining principle-quality measures that reliably predict end-to-end reconstruction remains an open challenge.


\vspace{-0.1em}
\subsection{Problem 2: Composition Ambiguity}
\vspace{-0.1em}

Executor dependence is substantial under naive principles, consistent with a composition gap. With the same principles, transparent executors underperform the LLM judge on average (e.g., majority vote test $60.7\% \pm 1.9$ vs.\ LLM test $64.5\% \pm 2.0$; Table~\ref{tab:reconstruction}) and disagree with it on a nontrivial fraction of comparisons (LLM vs.\ majority agreement $73.4\% \pm 4.2\%$; Table~\ref{tab:agreement_pair}). This indicates that the written constitution does not uniquely determine the effective decision rule: implicit composition logic supplied by the executor materially affects predictions.

With ICAI+ principles, both the performance gap and inter-executor disagreement narrow. Majority vote reaches $64.6\% \pm 2.6$ versus $65.5\% \pm 2.1$ for the LLM judge on test (Table~\ref{tab:reconstruction}), and LLM versus majority agreement increases to $78.4\% \pm 4.1$ (Table~\ref{tab:agreement_pair}). An ablation demonstrating refinement improves over only filtering can be seen in Appendix~\ref{sec:principle-refinement-ablation}. Overall, refinement makes principles more operational and reduces the implicit composition logic supplied by the judge. Nevertheless, executor sensitivity remains: the same principles can yield meaningfully different outcomes under different executors, so constitutions should be evaluated as \emph{constitution--executor systems}, not standalone texts.

\paragraph{Why do trees underperform?}

Although decision trees can represent arbitrary compositions of principle votes, representing a counting-style rule like majority vote is exponentially inefficient in the number of principles: exact majority requires exponentially many leaves. This creates a representation-efficiency mismatch under depth/leaf constraints, biasing learned trees toward sparse, priority-like rules. We formalize this in Lemma~\ref{lem:majority-tree-size} and prove it in Appendix~\ref{app:tree_proof}.

\begin{lemma}[Majority has exponential decision-tree size]
\label{lem:majority-tree-size}
Let $d$ be odd and let $\mathrm{Maj}_d:\{-1,+1\}^d\to\{-1,+1\}$ be binary majority,
\(
\mathrm{Maj}_d(x) = \mathrm{sign}\!\Big(\sum_{j=1}^d x_j\Big).
\)
Any (deterministic) decision tree that computes $\mathrm{Maj}_d$ exactly must have at least $2^{(d-1)/2}$ leaves.
Consequently, any decision tree that computes the ternary majority rule $g_{\mathrm{MV}}$ on $\{-1,0,+1\}^d$ exactly must have at least $2^{(d-1)/2}$ leaves.
\end{lemma}

Lemma~\ref{lem:majority-tree-size} shows that \emph{exactly} representing a counting-style rule like majority requires exponentially many leaves in $d$ (and therefore an exponentially large tree in the worst case). In contrast, priority vote is easy: it is representable by a depth-$d$ tree with only $2d+1$ leaves (follow the fixed order; on $\pm 1$ output immediately; on $0$ continue). Thus, when we constrain trees to remain low-complexity, either for interpretability (small depth) or for statistical reasons (to avoid overfitting on noisy vote features), tree learning has a strong inductive bias toward \emph{priority-like} and against \emph{majority-like} aggregation.

\input{figures/coverage_accuracy.tex}
\input{tables/performance.tex}
\input{tables/inter_executor_agreement.tex}

\paragraph{Connection to our empirical setting}
Learned trees and boosted ensembles are typically regularized (e.g., by limiting maximum depth / number of leaves / ensemble complexity), as otherwise they become uninterpretable and prone to overfitting. Under such constraints, the model class may simply be too small to express an additive ``sum-of-many-weak-signals'' composition like majority, even if it is a good fit for the data. This may explain why explicit majority vote can outperform tree-based executors in Table~\ref{tab:reconstruction}: majority implements a composition of \emph{all} principle votes, whereas trees tend to implement a sparse, precedence-driven rule that conditions on only a few principles along any path.

Note that LightGBM (an ensemble of trees of principles, rather than a single decision tree of principles) performs roughly on par with majority vote (Table~\ref{tab:reconstruction}), plausibly because an ensemble can approximate threshold-like ``counting'' rules with much lower effective complexity than a single small tree. However, it does not reliably \emph{outperform} majority vote, suggesting that in our setting, tree ensembles are unable to discover, or possibly represent, a substantially better composition rule than majority-vote aggregation. The remaining gap to LLM-as-a-judge execution therefore seems less like a simple failure of choosing the right transparent model, and more like evidence that the LLM judge is adding a complex or opaque composition policy.



\subsection{Problem 3: Cross-LLM Transferability}\label{section:results.transferability}

We find that applying the \emph{same} principle text is not reliably stable under model substitution. Different \LLM{} annotators disagree both on whether a principle is applicable (coverage agreement; Table~\ref{tab:cross_llm_agreement_icai_pair}) and on vote direction conditional on applicability (vote agreement; Table~\ref{tab:cross_llm_agreement_icai_pair}). For naive ICAI principles, mean coverage agreement is $82.8\% \pm 1.1$, while vote agreement is $72.2 \pm 2.0$, indicating that even when two models agree a principle applies, they often disagree on which response it favors. For comparison, intra-LLM agreement (the same model evaluating the same principles across repeated runs) yields coverage agreement of $87.0\% \pm 2.1$ and vote agreement of $80.5\% \pm 3.2$, indicating that a substantial portion of the cross-LLM disagreement reflects genuine differences in how models interpret principle text rather than measurement noise.

ICAI+ does not significantly improve this instability. Vote agreement increases to $73.0\% \pm 2.1$, but coverage agreement decreases to $80.9\% \pm 1.1$ (Table~\ref{tab:cross_llm_vote_agreement_icaiplus}). This is consistent with ICAI+ producing principles which are less semantically ambiguous, but thus more complex --- as seen in Table~\ref{tab:cross_llm_vote_agreement_icaiplus}, the drop in coverage agreement is largely due to GPT 5.4-nano, the weakest model. Taken together, these results suggest that constitutions are not intrinsically model-agnostic artifacts: their operational meaning depends on the model used to interpret them. This limits transferability and interpretability, and motivates reporting the full reconstruction stack (discoverer/annotator/executor) when using constitutions for auditing or evaluation.

At the discovery level, constitutions induced by different models also diverge semantically (Appendix~\ref{sec:llm-constitution-comparison}), reinforcing the conclusion that both \emph{what} principles are induced and \emph{how} they are interpreted can be model-dependent.

%% file: figures/principle_clustering.tex
\begin{figure}[t]
  \centering
  \begin{subfigure}{0.49\textwidth}
    \centering
    \includegraphics[width=\linewidth,height=0.30\textheight,keepaspectratio]{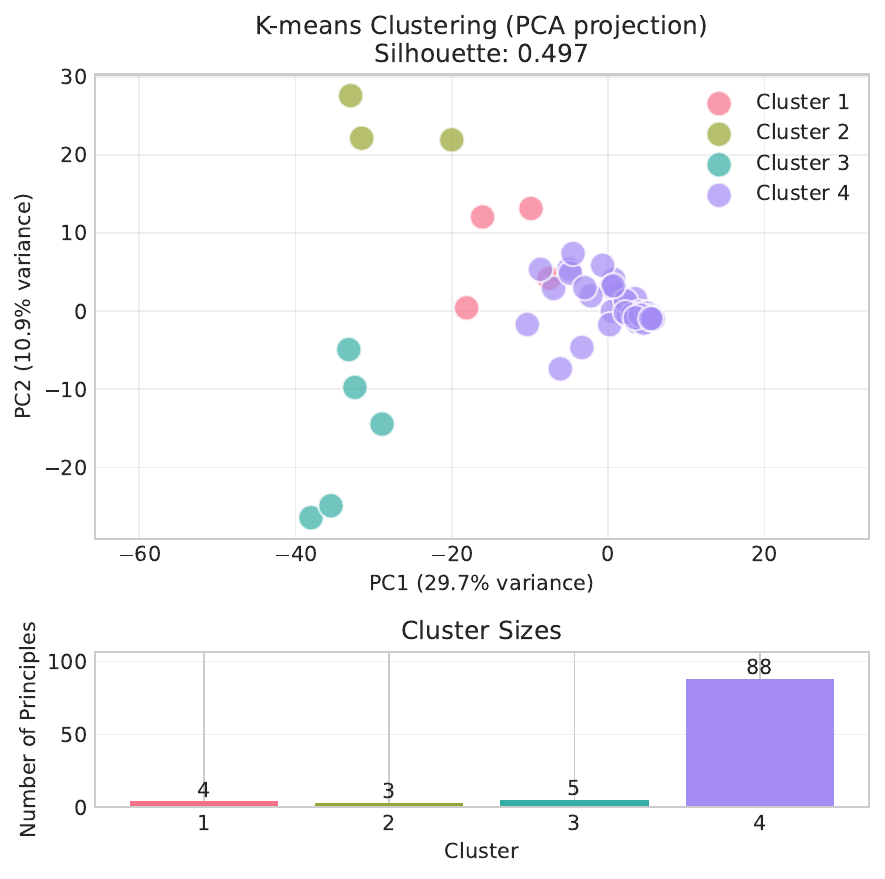}
    \caption{Naive principle generation}
    \label{fig:kmeans_naive}
  \end{subfigure}\hfill
  \begin{subfigure}{0.49\textwidth}
    \centering
    \includegraphics[width=\linewidth,height=0.30\textheight,keepaspectratio]{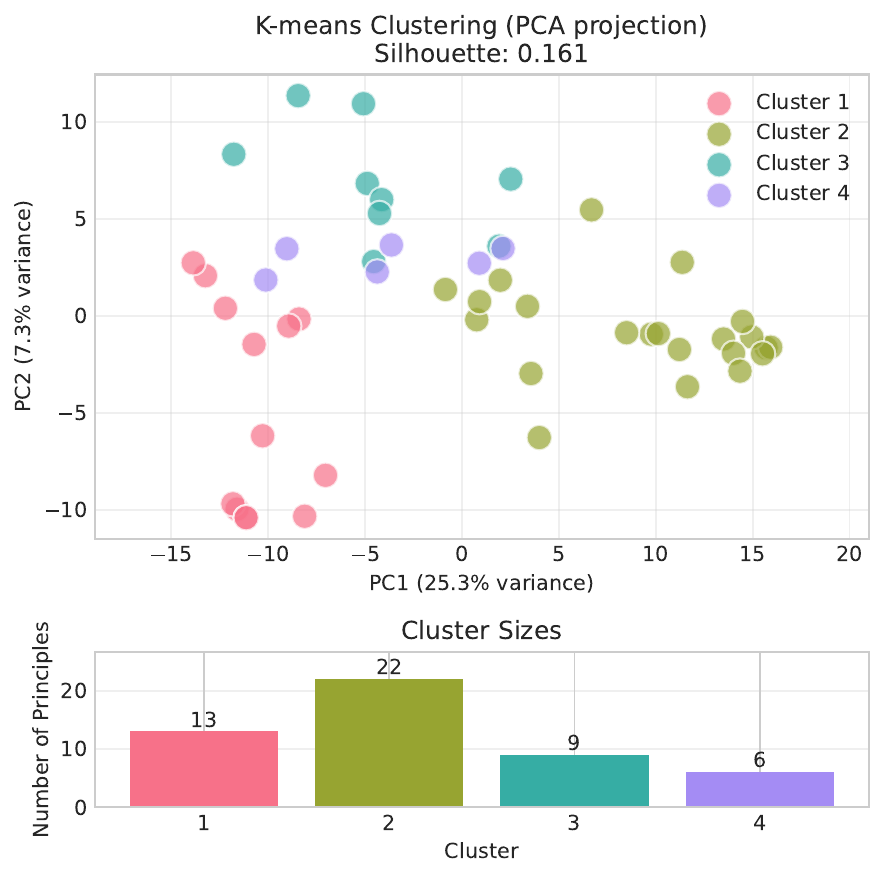}
    \caption{Improved principle generation}
    \label{fig:kmeans_improved}
  \end{subfigure}
	\caption{
		K-means clustering of generated principles, between methods, on AlpacaEval using DeepSeek v3.1 Chat. Cluster sizes show that the improved principle generation procedure produces more varied principles, whereas the naive principle generation produces a single dominant cluster.
	}
    \vspace{-0.5em}
  \label{fig:kmeans}
\end{figure}

%% file: figures/coverage_accuracy.tex
\begin{figure}[t]
  \centering
  \begin{subfigure}{0.85\textwidth}
    \centering
    \includegraphics[width=\linewidth,height=0.33\textheight,keepaspectratio]{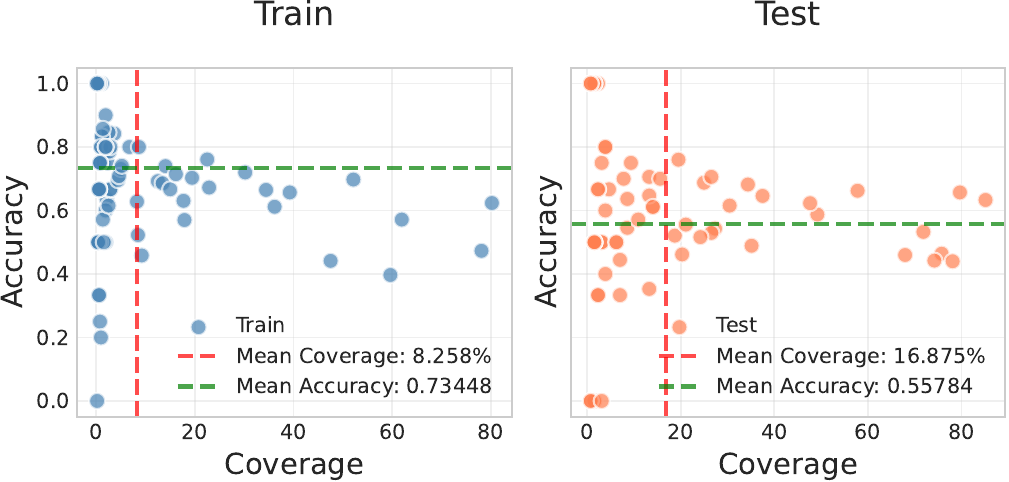}
    \caption{Naive principle generation}
    \label{fig:coverage_accuracy_icai}
    \vspace{1em}
  \end{subfigure}
  \begin{subfigure}{0.85\textwidth}
    \centering
    \includegraphics[width=\linewidth,height=0.33\textheight,keepaspectratio]{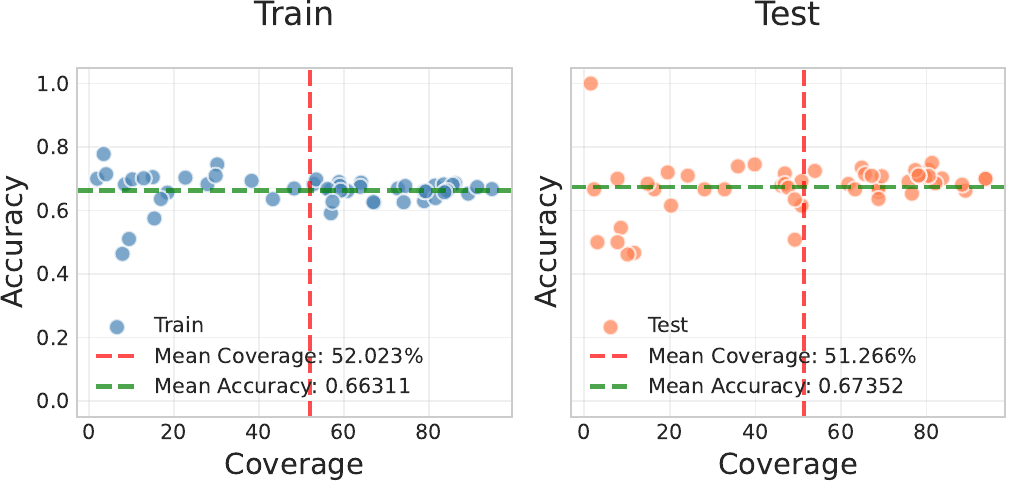}
    \caption{Improved principle generation}
    \label{fig:coverage_accuracy_icaiplus}
    \vspace{0.5em}
  \end{subfigure}
  \caption{Coverage and accuracy of generated principles on AlpacaEval using DeepSeek v3.1 Chat. ICAI+ yields higher-coverage and moderately higher-accuracy principles.}
  \label{fig:coverage_accuracy}
  \vspace{-0.5em}
\end{figure}

%% file: tables/performance.tex
\begin{table}[b]
  \vspace{-0.5em}
  \centering
  \small
  \setlength{\tabcolsep}{5pt}
  \renewcommand{\arraystretch}{1.1}
  \caption{Performance (\%) across executors with the same principles. ICAI+ moderately improves all methods, and closes the gap between opaque and transparent executors. Full results in Appendix~\ref{sec:full-results}.}
  \label{tab:reconstruction}
  \vspace{1em}
  \begin{tabular}{lcccc}
    \toprule
    & \multicolumn{2}{c}{Performance (train)} & \multicolumn{2}{c}{Performance (test)} \\
    \cmidrule(lr){2-3}\cmidrule(lr){4-5}
    Aggregation method & ICAI & ICAI+ & ICAI & ICAI+ \\
    \midrule
    LLM           & $66.4 \pm 1.6$ & $67.2 \pm 1.7$ & $64.5 \pm 2.0$ & $65.5 \pm 2.1$ \\
    Majority vote & $69.5 \pm 1.8$ & $70.2 \pm 1.8$ & $60.7 \pm 1.9$ & $64.6 \pm 2.6$ \\
    Priority vote & $65.2 \pm 2.1$ & $67.6 \pm 1.9$ & $60.1 \pm 2.0$ & $62.5 \pm 2.2$ \\
    LightGBM      & $68.8 \pm 2.5$ & $70.4 \pm 1.7$ & $61.3 \pm 2.2$ & $64.6 \pm 2.1$ \\
    Decision Tree & $63.2 \pm 2.8$ & $67.2 \pm 2.1$ & $58.0 \pm 3.0$ & $62.3 \pm 2.1$ \\
    \bottomrule
  \end{tabular}
\end{table}

%% file: tables/inter_executor_agreement.tex
\begin{table}[b]
  \vspace{-0.75em}
  \centering
  \setlength{\tabcolsep}{2pt}
  \renewcommand{\arraystretch}{1.10}
  \caption{Test agreement (\%) across executors with the same principles. ICAI+ moderately improves all agreements. MV=majority vote, PV=priority vote, GBM=LightGBM, DT=decision tree.}
  \label{tab:agreement_pair}
  \begin{subtable}[t]{0.48\textwidth}
    \centering
    \caption{ICAI (naive) principles}
    \label{tab:agreement_naive}
	\adjustbox{max width=\linewidth}{
	\begin{tabular}{lcccc}
      \toprule
      & LLM & MV & PV & GBM \\
      \midrule
      LLM & -- & -- & -- & -- \\
      MV  & $73.4 \pm 4.2$ & -- & -- & -- \\
      PV  & $71.0 \pm 4.2$ & $83.4 \pm 3.0$ & -- & -- \\
      GBM & $72.8 \pm 4.9$ & $80.4 \pm 3.9$ & $80.6 \pm 4.0$ & -- \\
      DT  & $67.3 \pm 6.0$ & $74.7 \pm 5.4$ & $83.6 \pm 7.7$ & $74.4 \pm 7.3$ \\
      \bottomrule
	\end{tabular}
	}
  \end{subtable}\hfill
  \begin{subtable}[t]{0.48\textwidth}
    \centering
    \caption{ICAI+ (improved) principles}
    \label{tab:agreement_icaiplus}
	\adjustbox{max width=\linewidth}{
	\begin{tabular}{lcccc}
      \toprule
      & LLM & MV & PV & GBM \\
      \midrule
      LLM & -- & -- & -- & -- \\
      MV  & $78.4 \pm 4.1$ & -- & -- & -- \\
      PV  & $75.2 \pm 4.3$ & $85.7 \pm 3.0$ & -- & -- \\
      GBM & $76.5 \pm 4.5$ & $89.9 \pm 3.4$ & $87.8 \pm 3.2$ & -- \\
      DT  & $74.8 \pm 4.4$ & $83.7 \pm 3.7$ & $95.5 \pm 3.3$ & $87.1 \pm 3.5$ \\
      \bottomrule
    \end{tabular}
	}
  \end{subtable}
\end{table}

%% file: sections/discussion.tex
\subsection{Analysis}

Our results demonstrate that a constitution alone is not a complete decision rule.
It lists what principles matter, but it does not say how to use those principles when several apply, when they conflict, or when none of them quite fit.
That missing step matters empirically: with the same constitution, different executors can produce different labels.
Reconstruction accuracy should therefore be read as a property of the constitution--executor system, not of the principle text alone.

This also changes how to interpret the strong performance of constitutions executed with an LLM-as-a-judge.
When principles are narrow, repetitive, or vague, a transparent executor can only aggregate the explicit principle-level votes.
An LLM judge has more freedom: it can decide which principles matter, smooth over vague wording, import missing criteria, and resolve conflicts using its own background preferences.
This can improve reconstruction, but it reduces the explanatory power of the constitution -- some of the accuracy belongs to the judge's implicit policy, not the constitution itself.

The results with ICAI+'s improved principles support this interpretation.
Under naive ICAI principles, the LLM judge has a clear advantage in reconstruction accuracy over majority vote, but after refinement, the performance gap is well within error bars, and agreement increases.
This suggests that the original gap was not only due to the executor being stronger, but also due to weaknesses in the principle list.
The remaining disagreement shows that the problem is not solved, but the ICAI+ demonstrates that choosing a more auditable executor is not necessarily always at the expense of reconstruction accuracy, and gives a useful check for future discovery methods: after principles are discovered, how much do their decisions still depend on which executor applies them?

Cross-\LLM{} transferability exposes the same problem one step earlier in the stack.
While executor sensitivity shows that a constitution does not specify how principle-level votes should be composed, cross-\LLM{} disagreement shows that the principle text may not even determine the votes themselves.
Two annotator models can read the same principle and disagree about whether it applies or which response it favors.
In that case, the principle is not a full explanation of the data; its meaning is partly determined by the annotator that applies it.
We hope that a broader impact of this work will be demonstrating caution is necessary when interpreting LLM artifacts as true explanations of underlying data, and constitutions should not be reported as standalone: the discoverer, annotator, and executor all affect what they mean operationally.
We also note that even with improved metrics, using LLMs to explain human decisions is prone to generating bias and a misleading sense of understanding.

\paragraph{Implications for forward preference learning}

Pairwise labels used in RLHF and DPO are usually treated as simple targets for training.
Under our framing, each label is instead the endpoint of a hidden decision process: criteria were applied, conflicts were handled, and a final choice was made.
Preference reconstruction is useful because it lets us inspect, before training, what kind of rule a preference learner would be asked to imitate.

This enables us to audit the preferences our model tries to learn during training.
If a reconstructed preference changes when we swap the executor, or when another annotator model applies the same principles, then the current constitution has not pinned down the decision rule, and this inconsistency applies to constitutions used for training just as it applies to our setting.
These cases should not be dismissed as ordinary noise; they mark places where the supervision is underspecified: a principle may be too vague, or the constitution may fail to say how competing criteria should be traded off.

A practical use of these diagnostics is therefore to find the brittle parts of a preference dataset or candidate constitution before they are used downstream.
Unstable examples can be flagged for human-led principle revision, or closer inspection.
Preference labels can hide executor-specific choices, and reconstruction can expose some of those choices before they used in training.

%% file: sections/conclusion.tex
\section{Conclusion}

Preference datasets are widely used to train and evaluate \LLMs{}, but they are difficult to audit because they encode individual decisions between options and not the underlying rationale. ICAI-style constitutions aim to improve legibility by compressing preferences into a short list of natural-language principles. We argued that this is only a partial solution: a flat constitution is not yet an executable rule, because it leaves principle composition implicit and therefore executor-dependent.

Empirically, we identified three open problems: (i) principle quality is hard to measure using per-principle coverage and accuracy alone; (ii) holding principles fixed, different executors often disagree, indicating ambiguity in principle composition; and (iii) constitutions are not reliably transferable across model stacks, both at the level of discovered principle sets and at the level of applying a given principle. We showed that improving the principle-generation procedure with ICAI+ narrows executor dependence and makes transparent executors competitive with LLM judges. Our results motivate evaluating constitutions as \emph{constitution--executor systems} and treating executor robustness and transferability as core requirements for auditable preference reconstruction.

\paragraph{Limitations}

We focused on reconstruction accuracy as the primary quantitative metric of a constitution; while we discuss other desiderata for constitutions qualitatively (e.g., human interpretability, redundancy), they are difficult to measure. Additionally, ICAI+ improves coverage and reduces ambiguity but does not eliminate executor dependence entirely. More principled methods for discovering unambiguous, transferable principles remain an open problem.

%% file: sections/acknowledgements.tex
\section*{Acknowledgments}

Eleanor Clifford was supported by the Engineering and Physical Sciences Research Council. Arduin Findeis was supported by a University of Cambridge School of Physical Sciences Award, the UKRI Centre for Doctoral Training in Application of Artificial Intelligence to the study of Environmental Risks (reference EP/S022961/1), a Trinity Hall PhD Tenth Term Funding Award and a Cambridge Philosophical Society Research Studentship.

%% file: appendices/code.tex
\section{Reproducibility}

All code and other resources required to reproduce the results in this paper is released open source and can be found at \url{https://r.ecc.im/OpenProblemsInCAI}.

%% file: appendices/hyperparameters.tex
\section{Summary of Algorithm~\ref{alg:icaiplus} hyperparameters}\label{sec:icaiplus-parameters}

\begin{itemize}
	\item \textbf{minimum annotations $k_{\textrm{min}}$.} Coverage estimation is unreliable on a very small number of samples, so principles are not discarded while the number of annotations is less than $k_{\textrm{min}}$. A higher value is required to select useful but low-coverage principles, but risks wasting API calls evaluating bad principles.

	\item \textbf{maximum annotations $k_{\textrm{max}}$.} The number of annotations required before a principle is accepted.

	\item \textbf{minimum applicable votes $r_{\textrm{min}}$.} Similar to $k_{\textrm{min}}$, but since accuracy is only estimated on applicable votes, it is a separate parameter.

	\item \textbf{minimum coverage $c_{\textrm{min}}$.} Minimum acceptable coverage for a principle to not be discarded.

	\item \textbf{minimum accuracy $a_{\textrm{min}}$.} Minimum acceptable accuracy for a principle to not be discarded.

	\item \textbf{refinement accuracy threshold $a_r$.} Minimum acceptable accuracy for a principle to be accepted without refinement.

	\item \textbf{target count $d$.} The target number of total principles to generate.
\end{itemize}

\section{Hyperparameters}\label{sec:hyperparameters}

ICAI was configured to generate 400 candidate principles per dataset and then cluster them down to 100 via $k$-means over cosine similarity.
ICAI+ was configured with the hyperparameters seen in Table~\ref{tab:icaiplus_hyperparams}.
In all experiments, it was configured to generate a total of 50 principles rather than 100.
This was to bring the total compute cost of each into a similar range, since ICAI+ uses significantly more compute in principle discovery generating principles and then discarding them.

Hyperparameters for the datasets used can be seen in Table~\ref{tab:dataset-params}.

\begin{table}[H]
	\centering
	\caption{Hyperparameters used with Algorithm~\ref{alg:icaiplus}}
	\label{tab:icaiplus_hyperparams}

	\begin{tabular}{lllllllll}
		\toprule
		LLM                 & dataset     & $k_{\textrm{min}}$  & $k_{\textrm{max}}$  & $r_{\textrm{min}}$  & $c_{\textrm{min}}$  & $a_{\textrm{min}}$  & $a_r$  & $d$  \\
		\midrule
		DeepSeek v3.1 Chat  & all         & 20                  & 100                 & 4                   & 0.04                & 0.55                & 0.65   & 50   \\
		GPT 4o              & all         & 50                  & 100                 & 4                   & 0.02                & 0.55                & 0.65   & 50   \\
		GPT 4o-mini         & all         & 50                  & 100                 & 4                   & 0.02                & 0.55                & 0.65   & 50   \\
		GPT 5.4-mini        & all         & 50                  & 100                 & 4                   & 0.02                & 0.55                & 0.65   & 50   \\
		GPT 5.4-nano        & all         & 50                  & 100                 & 4                   & 0.02                & 0.55                & 0.65   & 50   \\
		Gemini 2.5 Flash    & AlpacaEval  & 50                  & 100                 & 4                   & 0.02                & 0.60                & 0.70   & 50   \\
		Gemini 2.5 Flash    & Arena       & 20                  & 100                 & 4                   & 0.04                & 0.55                & 0.65   & 50   \\
		Gemini 2.5 Flash    & PRISM       & 20                  & 100                 & 4                   & 0.04                & 0.55                & 0.65   & 50   \\
		\bottomrule
	\end{tabular}
\end{table}

\begin{table}[H]
	\centering
	\caption{Dataset sizes}
	\label{tab:dataset-params}

	\begin{tabular}{lcc}
		\toprule
		Dataset     & Training samples  & Test samples  \\
		\midrule
		AlpacaEval  & 520  & 128  \\
		Arena       & 800  & 400  \\
		PRISM       & 800  & 400  \\
		\bottomrule
	\end{tabular}
\end{table}

%% file: appendices/principle_refinement_ablation.tex
\section{Ablation: Principle Refinement}\label{sec:principle-refinement-ablation}

In order to investigate whether principle refinement improves performance over just filtering, we took ICAI+'s final generated principles, randomly selected 12 that were either (a) directly generated without refinement or (b) refined at least once (both sets were filtered), and compared the average performance of a 10-principle majority vote from each set, across all datasets and models (12 was the largest equal size possible). The results can be seen in \ref{fig:principle-refinement-ablation}.

\begin{figure}[H]
	\centering
	\includegraphics[width=0.6\textwidth]{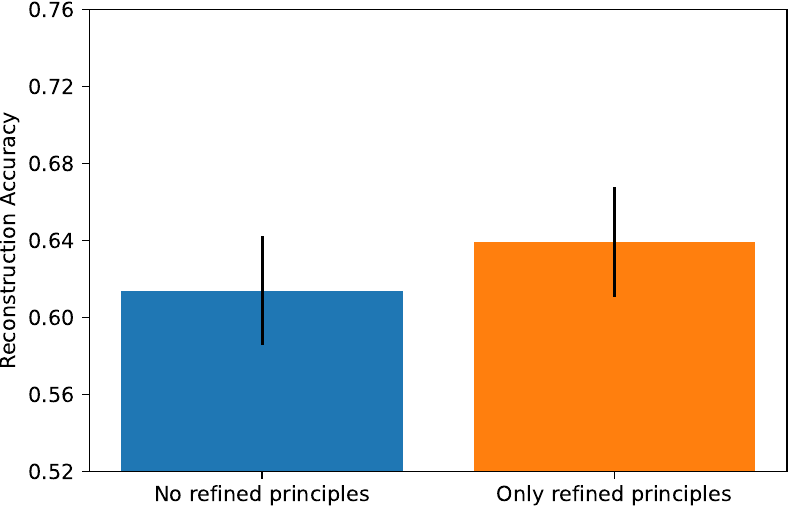}
	\caption{Performance of 10-principle majority vote executor with only refined and only unrefined principles.}
	\label{fig:principle-refinement-ablation}
\end{figure}

Reconstruction accuracy with refined principles is significantly higher than reconstruction accuracy with only unrefined principles. As we can see from Figure~\ref{fig:constitution-size-ablation}, aggregate results are generally relatively stable across constitution sizes, so this result should be expected to generalise to larger constitutions.

%% file: appendices/constitution_size_ablation.tex
\section{Ablation: Constitution Size}\label{sec:constitution-size-ablation}

In order to determine a reasonable constitution size, majority vote reconstruction accuracy was evaluated against constitution size. Majority vote is the highest-performing executor that would not require further compute to evaluate against constitution size (unlike LLM-as-a-judge). The results can be seen in Figure~\ref{fig:constitution-size-ablation}.

\begin{figure}[H]
	\centering
	\includegraphics[width=\textwidth]{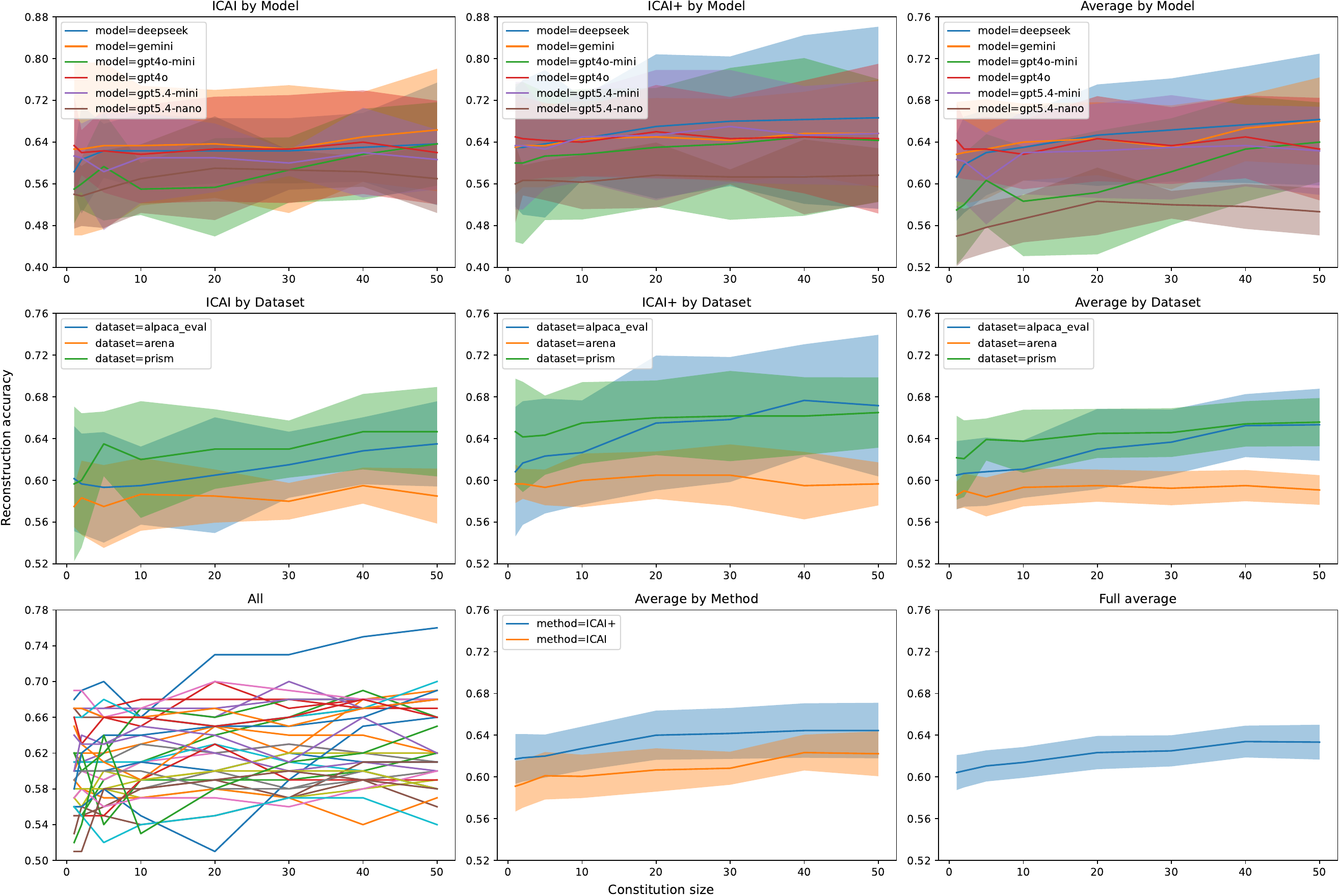}
	\caption{Performance of majority vote executor against number of principles in the constitution.}
	\label{fig:constitution-size-ablation}
\end{figure}

Improvement in reconstruction accuracy of ICAI+ is marginal after 20-30 principles, and below this there is significant noise in some datasets and models. On Chatbot Arena, performance actually decreases above 20-30 principles, because lower quality principles are added which can outvote the better ones. ICAI+ maintains a comfortable lead over ICAI at all sizes of constitution. Based on this, 30 was selected as a reasonable constitution size.

%% file: appendices/lemma_proof.tex
\section{Proof of Lemma \ref{lem:majority-tree-size}}
\label{app:tree_proof}

\begin{proof}[Proof sketch]
A leaf of a decision tree corresponds to a partial assignment of votes that yields a constant output of that tree. Consider any leaf labeled $+1$. If this leaf fixes fewer than $(d+1)/2$ coordinates to $+1$, then the remaining unfixed coordinates can be set to $-1$, yielding an input with strict $-1$ majority, contradicting correctness. Therefore every $+1$-leaf must fix at least $(d+1)/2$ variables to $+1$.

Such a leaf is consistent with at most $2^{d-(d+1)/2}=2^{(d-1)/2}$ inputs for which the result is $+1$, since the remaining variables (if any) can vary freely. But $\mathrm{Maj}_d(x)=+1$ holds for exactly $2^{d-1}$ inputs. Hence the number of $+1$-leaves is at least
\(
\frac{2^{d-1}}{2^{(d-1)/2}} = 2^{(d-1)/2}.
\)
For the ternary claim, restrict $g_{\mathrm{MV}}$ to the binary sub-cube $\{-1,+1\}^d\subset \{-1,0,+1\}^d$, on which it coincides with $\mathrm{Maj}_d$.
\end{proof}

%% file: appendices/compute.tex
\section{Compute}\label{sec:compute}

\newcommand{\TotalLLMCalls}{\textbf{[TBA]}}
\newcommand{\TotalInputTokens}{\textbf{[TBA]}}
\newcommand{\TotalOutputTokens}{\textbf{[TBA]}}
\newcommand{\TotalAPICostUSD}{\textbf{[TBA]}}

For the results in the paper, ICAI used approximately 100 LLM calls per run in discovery and 10000 in annotation, ICAI+ used approximately 2000 per run in discovery and 5000 in annotation. LLM executors used approximately 1000 LLM calls per run for both. Across all experiments, the total number of LLM calls was approximately 350000.

The average cost per token on OpenRouter across the 6 models tested was \$0.68/M input tokens, \$3.27/M output tokens. We estimate that the average number of input tokens per API call was approximately 500, and output tokens was approximately 100. In total, the API cost of the results in the paper is estimated to be \$250 USD.

Across the entire development cycle of the paper, the total API cost is estimated to be \$1500 USD.

%% file: appendices/example_constitutions_across_methods.tex
\section{Example Constitution (ICAI vs. ICAI+)}\label{sec:example_constitutions}

Figure \ref{fig:example_constitutions} shows the first 10 principles selected by Algorithm~\ref{alg:greedy_constitution} from ICAI and ICAI+ run on AlpacaEval using Gemini 2.5 Flash. Qualitatively, we can see that the ICAI principles are less diverse and more semantically ambiguous.

\begin{figure}[H]
	\begin{subfigure}{.3525\textwidth}
		\begin{center}
			\begin{minipage}[t]{\linewidth}
				\input{figures/constitutions/naive_alpaca_eval_deepseek_by_mv.tex}
            \end{minipage}
		\end{center}
		\caption{Naive principle generation}
		\label{fig:example_constitution_icai}
	\end{subfigure}
	\hfill
	\begin{subfigure}{.6\textwidth}
		\begin{center}
			\begin{minipage}[t]{\linewidth}
				\input{figures/constitutions/improved_alpaca_eval_deepseek_by_mv.tex}
			\end{minipage}
		\end{center}
		\caption{Improved principle generation}
		\label{fig:example_constitution_icaiplus}
	\end{subfigure}
	\caption{First 10 principles selected by Algorithm~\ref{alg:greedy_constitution} from generation on AlpacaEval using DeepSeek v3.1 Chat}
	\label{fig:example_constitutions}
\end{figure}

%% file: figures/constitutions/naive_alpaca_eval_deepseek_by_mv.tex
\setlist[enumerate]{leftmargin=*}
\raggedright
\begin{enumerate}
\item
  Select the response that uses more specific and professional language.
\item
  Select the response that uses bullet points for clarity.
\item
  Select the response that includes specific examples of benefits.
\item
  Select the response that avoids unnecessary spacing in the output.
\item
  Select the response that provides a conclusive and meaningful
  resolution.
\item
  Select the response that includes character development and
  progression.
\item
  Select the response that includes more specific hashtags.
\item
  Select the response that acknowledges historical complexity and
  uncertainty
\item
  Select the response that emphasizes mood-setting elements.
\item
  Select the response that provides a concrete code example.
\end{enumerate}

%% file: figures/constitutions/improved_alpaca_eval_deepseek_by_mv.tex
\setlist[enumerate]{leftmargin=*}
\raggedright
\begin{enumerate}
\item
  Select the response that provides the most correct and complete
  essential information, unless excessive detail obscures accuracy.
\item
  Select the response that provides concrete feature specifics, except
  when the prompt requests brevity or a simplified format.
\item
  Select the response that provides actionable steps for implementing
  instructions, unless the prompt explicitly requests general concepts
  or product recommendations.
\item
  Select the response that precisely follows quantitative constraints,
  unless it introduces factual inaccuracies or omits key elements.
\item
  Select the response that offers precise measurements and clear
  instructions for reliable execution.
\item
  Select the response that uses formatting like bullet points and
  numbered steps for enhanced readability.
\item
  Select the response that correctly interprets numerical input
  constraints.
\item
  Select the response that uses metaphors only when they creatively
  clarify or simplify complex concepts, not for straightforward factual
  descriptions.
\item
  Select the response that provides accurate, actionable instructions,
  unless correctness of the instructions is itself the primary focus,
  then prioritize accuracy.
\item
  Select the response that provides accurate and complete information,
  addressing all explicit sub-questions in the prompt.
\end{enumerate}

%% file: appendices/greedy_vs_naive.tex
\section{Greedy majority vote selection compared to metric-only selection of principles}\label{sec:gmv-metric-comparison}

This appendix details the greedy procedure used to select principles for flat constitutions and compares it to metric-only selection.

\begin{algorithm}[b]
\caption{Greedy selection of a flat constitution (maximize MV train accuracy)}
\label{alg:greedy_constitution}
\begin{algorithmic}[1]
\STATE \textbf{Input:} candidate principles $\mathcal{P}$; train votes $V_{i,p}\in\{\text{A},\text{B},\text{N/A}\}$; labels $y_i$; budget $K$ (e.g., 30)
\STATE \textbf{Defs:} $\mathrm{MV}(\mathcal{C};V)$ predicts $\hat y_i$ by majority vote over $\{V_{i,p}:p\in\mathcal{C},\,V_{i,p}\neq \text{N/A}\}$ (ties broken deterministically, preferring response A)
\STATE \hspace{2.9em} $\mathrm{Acc}(\hat y,y)=\frac{1}{N}\sum_{i=1}^N \mathbf{1}[\hat y_i=y_i]$
\STATE $\mathcal{C}\leftarrow[\,]$
\FOR{$t=1$ to $K$}
  \STATE $p^\star \leftarrow \arg\max_{p\in \mathcal{P}\setminus \mathcal{C}}
  \ \mathrm{Acc}\!\Big(\mathrm{MV}(\mathcal{C}\cup\{p\};V),\, y\Big)$
  \STATE $\mathcal{C}\leftarrow \mathcal{C}\cup[p^\star]$
\ENDFOR
\STATE \textbf{return} $\mathcal{C}$
\end{algorithmic}
\end{algorithm}

Alternate options would be to (a) select the best performing principles based on metrics such as coverage and accuracy, without considering principle similarity or (b) to do as in Algorithm~\ref{alg:greedy_constitution} but maximise priority vote performance instead of majority vote. We find that both would result in constitutions with several very similar principles. In the case of (a), this is because if there are any other principles very similar to the first one, they will also perform similarly well. In the case of (b), this is because the procedure tends to fill in the coverage gaps by selecting further principles which are very similar but differ slightly in which comparisons they cover. In contrast, prioritising majority vote means prioritising not just principles which fill in the coverage gaps, but also ones which disagree with the previous principles on comparisons that would be annotated incorrectly.

\paragraph{Why optimize majority-vote performance for flat constitution selection?}
For flat constitutions we must select a subset of principles from a larger candidate pool.
We use Algorithm~\ref{alg:greedy_constitution} to greedily select $K$ principles that maximize \emph{majority-vote} reconstruction accuracy on the training vote annotations.
Below, we compare this procedure to selecting the top-$K$ principles purely by individual per-principle metrics (e.g., coverage/accuracy). We find that
across flat execution methods, reconstruction accuracy is well within the reported uncertainty, but metric-based selection tends to pick near-duplicates (highly similar principles with overlapping coverage), producing constitutions that are qualitatively less diverse and less informative.
In contrast, optimizing majority-vote performance encourages selecting \emph{complementary} principles that correct one another’s errors, resulting in constitutions that are qualitatively more interpretable (see Appendices \ref{sec:example_constitutions} and \ref{sec:gmv-metric-comparison}).
This also reinforces our broader finding: constitutions with similar end-to-end reconstruction can differ substantially in semantic content.

\paragraph{Selection strategy comparison.} Figure \ref{fig:example_constitutions_by_metric} shows the first 10 principles selected by ICAI's default principle quality metric (correct votes minus incorrect votes) in ICAI and ICAI+ run on AlpacaEval using Gemini 2.5 Flash. Qualitatively, we can see that these constitutions, particularly for ICAI+, are much more repetitive than the principles seen in Figure~\ref{fig:example_constitutions}. Tables \ref{tab:gmv_metric_selection_train} and \ref{tab:gmv_metric_selection_test} compare train and test accuracies between these principle selection metrics, across various executors. Note the interesting pattern here: on train data, Greedy-MV substantially outperforms Metric-only for majority vote (70.5 vs 66.8 for naive, 70.8 vs 69.4 for ICAI+), which makes sense since that's what it's optimizing. But on test, this gap largely disappears (61.5 vs 62.7 for naive, 65.8 vs 65.9 for ICAI+), suggesting the Greedy-MV selection may be slightly overfitting to the training votes while Metric-only generalizes comparably.

\begin{figure}[H]
	\begin{subfigure}{.31\textwidth}
		\begin{center}
			\begin{minipage}[t]{\linewidth}
				\input{figures/constitutions/naive_alpaca_eval_deepseek_by_metric.tex}
            \end{minipage}
		\end{center}
		\caption{Naive principle generation}
	\end{subfigure}
	\hfill
	\begin{subfigure}{.63\textwidth}
		\begin{center}
			\begin{minipage}[t]{\linewidth}
				\input{figures/constitutions/improved_alpaca_eval_deepseek_by_metric.tex}
			\end{minipage}
		\end{center}
		\caption{Improved principle generation}
	\end{subfigure}
	\caption{First 10 principles selected by ICAI's default principle quality metric (correct votes minus incorrect votes) from generation on AlpacaEval using DeepSeek v3.1 Chat}
	\label{fig:example_constitutions_by_metric}
\end{figure}

\begin{table}[H]
  \centering
  \footnotesize
  \setlength{\tabcolsep}{4pt}
  \renewcommand{\arraystretch}{1.12}
  \caption{Reconstruction accuracy (\%) for flat executors under two principle-selection strategies: \textbf{Greedy-MV} (GMV; Algorithm~\ref{alg:greedy_constitution}) and \textbf{Metric-only} selection (top-$K$ by individual principle metrics). GMV results match corresponding entries in Table~\ref{tab:reconstruction}. Abbreviations: MV=majority vote, PV=priority vote.}
  \label{tab:gmv_metric_selection}
  \begin{subtable}[t]{0.49\textwidth}
    \centering
    \caption{Train}
    \label{tab:gmv_metric_selection_train}
    \begin{tabular}{lcccc}
      \toprule
      & \multicolumn{2}{c}{ICAI} & \multicolumn{2}{c}{ICAI+} \\
      \cmidrule(lr){2-3}\cmidrule(lr){4-5}
      & GMV & Metric & GMV & Metric \\
      \midrule
      LLM   & 66.9 $\pm$ 1.9 & 66.7 $\pm$ 1.9 & 67.9 $\pm$ 1.9 & 67.8 $\pm$ 2.1 \\
      MV    & 70.5 $\pm$ 1.5 & 66.8 $\pm$ 2.1 & 70.8 $\pm$ 2.0 & 69.4 $\pm$ 1.8 \\
      PV    & 66.8 $\pm$ 1.6 & 66.1 $\pm$ 1.6 & 68.8 $\pm$ 1.7 & 68.9 $\pm$ 1.7 \\
      \bottomrule
    \end{tabular}
  \end{subtable}\hfill
  \begin{subtable}[t]{0.49\textwidth}
    \centering
    \caption{Test}
    \label{tab:gmv_metric_selection_test}
    \begin{tabular}{lcccc}
      \toprule
      & \multicolumn{2}{c}{ICAI} & \multicolumn{2}{c}{ICAI+} \\
      \cmidrule(lr){2-3}\cmidrule(lr){4-5}
      & GMV & Metric & GMV & Metric \\
      \midrule
      LLM   & 64.5 $\pm$ 2.0 & 64.1 $\pm$ 2.2 & 66.0 $\pm$ 2.2 & 66.4 $\pm$ 2.4 \\
      MV    & 61.5 $\pm$ 2.0 & 62.7 $\pm$ 2.1 & 65.8 $\pm$ 2.5 & 65.9 $\pm$ 2.8 \\
      PV    & 60.7 $\pm$ 2.2 & 61.4 $\pm$ 2.1 & 63.3 $\pm$ 2.3 & 63.5 $\pm$ 2.5 \\
      \bottomrule
    \end{tabular}
  \end{subtable}
\end{table}

%% file: figures/constitutions/naive_alpaca_eval_deepseek_by_metric.tex
\setlist[enumerate]{leftmargin=*}
\raggedright
\begin{enumerate}
\item
  Select the response that uses more specific and professional language.
\item
  Select the response that is more descriptive and elaborate.
\item
  Select the response that provides a conclusive and meaningful
  resolution.
\item
  Select the response that provides detailed justification.
\item
  Select the response that includes specific examples of benefits.
\item
  Select the response that provides more detailed actionable steps.
\item
  Select the response that uses clearer and more direct language.
\item
  Select the response that includes full phrases not just words.
\item
  Select the response that is scientifically accurate.
\item
  Select the response that uses a list format for clarity.
\end{enumerate}

%% file: figures/constitutions/improved_alpaca_eval_deepseek_by_metric.tex
\setlist[enumerate]{leftmargin=*}
\raggedright
\begin{enumerate}
\item
  Select the response with correct syntax and formatting, unless exact
  wording or coherent narrative continuity is required.
\item
  Select the response that conveys information accurately and
  completely, avoiding both unnecessary wordiness and insufficient
  detail.
\item
  Select the response that provides the most correct and complete
  essential information, unless excessive detail obscures accuracy.
\item
  Select the response that provides concrete examples or specific
  details, unless they are irrelevant, incorrect, or reduce clarity.
\item
  Select the response that provides specific examples or actionable
  details unless they are incomplete, misleading, or lack sufficient
  context.
\item
  Select the response that provides specific, accurate details, unless
  the response contains incorrect or unverified claims.
\item
  Select the response that answers all parts of the instruction
  explicitly and precisely.
\item
  Select the response that provides specific, concrete examples unless
  the request is hypothetical, imaginative, or emphasizes simplicity.
\item
  Select the response that provides comprehensive detail while avoiding
  speculative claims, unless summarizing existing knowledge.
\item
  Select the response that provides concrete examples or computational
  methods, unless brevity is requested or explanation is purely
  conceptual.
\end{enumerate}

%% file: appendices/example_constitutions_across_llms.tex
\section{Comparison of generated constitutions between LLMs}\label{sec:llm-constitution-comparison}

Figure~\ref{fig:example_constitutions_between_llms} illustrates that \emph{similar reconstruction accuracy does not imply a similar ``explanation''}.
We run naive ICAI on PRISM with two different model stacks (DeepSeek v3.1 Chat vs.\ GPT-4o-mini) and select a flat constitution using the same Greedy-MV procedure (Algorithm~\ref{alg:greedy_constitution}).
Despite achieving very similar end-to-end reconstruction under LLM-as-a-judge execution on the PRISM test set (DeepSeek: $70.9\%/65.7\%$ train/test; GPT-4o-mini: $68.1\%/66.0\%$), the resulting constitutions differ substantially in semantic content.

Qualitatively, the DeepSeek-derived principles in Figure~\ref{fig:example_constitutions_between_llms} emphasize interpersonal and narrative dimensions (e.g., expressing understanding of the user’s concern, focusing on relationship context, embodying a character’s traits), whereas the GPT-4o-mini-derived principles focus more on interaction management and practical framing (e.g., inviting elaboration, avoiding repetition/negative framing, giving maintenance or cashflow advice).
We chose this pair of runs specifically because their held-out reconstruction accuracies are closely matched, making clear that the constitution is \emph{not uniquely identified} by predictive performance alone: multiple semantically distinct principle sets can fit the same dataset nearly equally well.
This supports our broader claim that constitutions should be treated as \emph{model-dependent artifacts} of a particular reconstruction stack, rather than as model-agnostic explanations of the underlying preferences.

\begin{figure}[H]
	\begin{subfigure}{.49\textwidth}
		\begin{center}
			\begin{minipage}[t]{\linewidth}
				\input{figures/constitutions/naive_prism_deepseek_by_mv.tex}
            \end{minipage}
		\end{center}
		\caption{DeepSeek v3.1 Chat}
	\end{subfigure}
	\begin{subfigure}{.49\textwidth}
		\begin{center}
			\begin{minipage}[t]{\linewidth}
				\input{figures/constitutions/naive_prism_gpt4o-mini_by_mv.tex}
			\end{minipage}
		\end{center}
		\caption{GPT-4o mini}
	\end{subfigure}
	\caption{First 10 principles selected by Algorithm~\ref{alg:greedy_constitution} from naive ICAI on PRISM using two different discoverer models. Despite similar test reconstruction accuracy (DeepSeek: 65.7\%; GPT-4o-mini: 66.0\%), the resulting constitutions differ substantially in semantic content, illustrating that reconstruction performance does not uniquely identify an explanatory constitution.}
\label{fig:example_constitutions_between_llms}
\end{figure}

%% file: figures/constitutions/naive_prism_deepseek_by_mv.tex
\setlist[enumerate]{leftmargin=*}
\raggedright
\begin{enumerate}
\item
  Select the response that provides complete arguments without abrupt
  cutoff
\item
  Select the response that provides more detailed subtopics.
\item
  Select the response that acknowledges multiple stated and unstated
  reasons
\item
  Select the response that focuses on interpersonal relationships
  context.
\item
  Select the response that expresses understanding of the user's
  concern.
\item
  Select the response that follows instruction format precisely
\item
  Select the response that fully embodies the character's traits.
\item
  Select the response that covers a broader range of conflict areas.
\item
  Select the response that includes personal fulfillment as a goal.
\item
  Select the response that provides more detailed meal descriptions
\end{enumerate}

%% file: figures/constitutions/naive_prism_gpt4o-mini_by_mv.tex
\setlist[enumerate]{leftmargin=*}
\raggedright
\begin{enumerate}
\item
  Select the response that invites further elaboration on topics.
\item
  Select the response that provides more detailed cleaning steps.
\item
  Select the response that avoids mentioning previous recipe failures.
\item
  Select the response that includes practical cashflow management
  advice.
\item
  Select the response that acknowledges complexity in celebrity
  influence.
\item
  Select the response that acknowledges individual couple's needs.
\item
  Select the response that suggests regular maintenance for longevity.
\item
  Select the response that avoids repetition of previously read series.
\item
  Select the response that focuses on philosophical rather than personal
  narratives.
\item
  Select the response that discusses personalization and creative
  assistance.
\end{enumerate}

%% file: appendices/cross_llm_agreement.tex
\section{Tables: Cross-LLM Agreement}

Tables \ref{tab:cross_llm_agreement_icai_pair} and \ref{tab:cross_llm_agreement_icaiplus_pair} contain the results of the cross-LLM agreements computations referenced in the paper.

\begin{table}[H]
  \centering
  \footnotesize
  \setlength{\tabcolsep}{6pt}
  \renewcommand{\arraystretch}{1.10}
  \caption{Cross-LLM annotator agreement for \textbf{naive ICAI} principles on test data, reported as mean across 3 datasets $\pm$ 95\% confidence t-interval. Coverage agreement: whether a principle is applicable. Vote agreement: conditional on both models marking applicability. 10 principles sampled per discoverer model (60 total) and annotated by all models. Abbreviations: DS=DeepSeek v3.1 Chat, Gm=Gemini 2.5 Flash, G4o=GPT-4o, G4m=GPT-4o-mini, G5m=GPT-5.4-mini, G5n=GPT-5.4-nano.}
  \label{tab:cross_llm_agreement_icai_pair}
  \begin{subtable}[t]{\textwidth}
    \centering
    \caption{\textbf{Coverage agreement} (ICAI / naive). Mean $82.8 \pm 1.1$}
    \label{tab:cross_llm_coverage_agreement_icai}
    \begin{tabular}{lccccc}
        \toprule
             & DS              & Gm               & G4o             & G4m              & G5m              \\
        \midrule
        DS   & --              & --               & --              & --               & --               \\
        Gm   & $85.1 \pm 3.3$  & --               & --              & --               & --               \\
        G4o  & $86.0 \pm 2.0$  & $84.5 \pm 4.7$   & --              & --               & --               \\
        G4m  & $84.8 \pm 2.2$  & $84.8 \pm 3.6$   & $87.3 \pm 2.5$  & --               & --               \\
        G5m  & $80.9 \pm 6.6$  & $77.6 \pm 11.4$  & $84.3 \pm 8.4$  & $81.7 \pm 12.1$  & --               \\
        G5n  & $81.3 \pm 1.1$  & $82.8 \pm 0.2$   & $81.9 \pm 1.4$  & $82.9 \pm 1.4$   & $76.5 \pm 12.7$  \\
		\bottomrule
    \end{tabular}
  \end{subtable}\hfill
  \begin{subtable}[t]{\textwidth}
    \centering
    \caption{\textbf{Vote agreement} (ICAI / naive). Mean $72.2 \pm 2.0$}
    \label{tab:cross_llm_vote_agreement_icai}
    \begin{tabular}{lccccc}
        \toprule
             & DS               & Gm               & G4o              & G4m              & G5m              \\
        \midrule
        DS   & --               & --               & --               & --               & --               \\
        Gm   & $77.6 \pm 9.0$   & --               & --               & --               & --               \\
        G4o  & $79.9 \pm 7.2$   & $78.4 \pm 9.9$   & --               & --               & --               \\
        G4m  & $72.6 \pm 15.0$  & $72.4 \pm 15.6$  & $73.9 \pm 12.7$  & --               & --               \\
        G5m  & $75.9 \pm 6.4$   & $74.9 \pm 7.0$   & $76.0 \pm 7.7$   & $69.2 \pm 14.4$  & --               \\
        G5n  & $67.1 \pm 14.2$  & $66.0 \pm 17.0$  & $67.1 \pm 13.1$  & $65.4 \pm 13.7$  & $65.8 \pm 13.7$  \\
		\bottomrule
    \end{tabular}
  \end{subtable}
  \vspace{-1em}
\end{table}

\begin{table}[H]
  \centering
  \footnotesize
  \setlength{\tabcolsep}{6pt}
  \renewcommand{\arraystretch}{1.10}
  \caption{Cross-LLM annotator agreement for \textbf{ICAI+} principles on test data, as in Table~\ref{tab:cross_llm_agreement_icai_pair}}
  \label{tab:cross_llm_agreement_icaiplus_pair}
  \begin{subtable}[t]{\textwidth}
    \centering
    \caption{\textbf{Coverage agreement} (ICAI+). Mean $80.9 \pm 1.1$}
    \label{tab:cross_llm_coverage_agreement_icaiplus}
    \begin{tabular}{lccccc}
        \toprule
             & DS              & Gm              & G4o             & G4m             & G5m             \\
        \midrule
        DS   & --              & --              & --              & --              & --              \\
        Gm   & $82.7 \pm 3.8$  & --              & --              & --              & --              \\
        G4o  & $84.9 \pm 5.0$  & $84.4 \pm 5.5$  & --              & --              & --              \\
        G4m  & $81.8 \pm 5.3$  & $82.6 \pm 7.3$  & $84.5 \pm 4.8$  & --              & --              \\
        G5m  & $81.8 \pm 4.3$  & $79.5 \pm 6.2$  & $84.4 \pm 3.3$  & $80.3 \pm 3.3$  & --              \\
        G5n  & $76.4 \pm 3.1$  & $78.3 \pm 6.8$  & $78.5 \pm 5.3$  & $78.8 \pm 3.3$  & $75.2 \pm 2.3$  \\
		\bottomrule
    \end{tabular}
  \end{subtable}\hfill
  \begin{subtable}[t]{\textwidth}
    \centering
    \caption{\textbf{Vote agreement} (ICAI+). Mean $73.0 \pm 2.1$}
    \label{tab:cross_llm_vote_agreement_icaiplus}
    \begin{tabular}{lccccc}
        \toprule
             & DS               & Gm               & G4o             & G4m              & G5m              \\
        \midrule
        DS   & --               & --               & --              & --               & --               \\
        Gm   & $79.5 \pm 12.1$  & --               & --              & --               & --               \\
        G4o  & $80.0 \pm 11.6$  & $81.4 \pm 11.4$  & --              & --               & --               \\
        G4m  & $76.5 \pm 8.8$   & $77.2 \pm 12.2$  & $78.5 \pm 8.4$  & --               & --               \\
        G5m  & $72.7 \pm 10.0$  & $74.6 \pm 7.8$   & $73.9 \pm 9.0$  & $70.9 \pm 9.2$   & --               \\
        G5n  & $66.9 \pm 7.3$   & $66.5 \pm 11.1$  & $66.6 \pm 8.4$  & $65.3 \pm 11.9$  & $65.1 \pm 13.6$  \\
		\bottomrule
    \end{tabular}
  \end{subtable}
  \vspace{-1em}
\end{table}

%% file: appendices/full_results.tex
\section{Full results}\label{sec:full-results}

Tables~\ref{tab:full_results_icai} and~\ref{tab:full_results_icaiplus} report reconstruction accuracy (Definition \ref{def:reconstruction_accuracy}) for all dataset--model combinations under naive ICAI and ICAI+ principle generation, respectively. We also include the mean coverage and accuracy of the generated principles for each configuration. For fair comparison, coverage and accuracy are measured only over the 30 principles selected for use in the flat constitutions. Abbreviations: MV = majority vote, PV = priority vote, GBM = LightGBM, DT = decision tree.

\begin{table}[H]
  \centering
  \footnotesize
  \setlength{\tabcolsep}{4pt}
  \renewcommand{\arraystretch}{1.10}
  \caption{Full reconstruction accuracy (\%) for ICAI (naive) principles across all datasets and models, along with coverage and accuracy (\%) of generated principles, presented as mean (median) $\pm$ standard deviation. Note that most principle metric distributions are highly skewed, and the standard deviation reported should not be interpreted as an error bar.}

  \label{tab:full_results_icai}
  \begin{subtable}[t]{\textwidth}
    \centering
    \caption{Train}
    \begin{tabular}{llllccccc}
      \toprule
      Dataset     & Model               & Coverage              & Accuracy              & LLM   & MV    & PV    & GBM   & DT \\
      \midrule
      AlpacaEval  & DeepSeek v3.1 Chat  & 7.0 (0.4) $\pm$ 16.6    & 88.0 (100.0) $\pm$ 17.9  & 67.3  & 70.0  & 66.9  & 67.5  & 59.6  \\
      AlpacaEval  & Gemini 2.5 Flash    & 8.8 (4.0) $\pm$ 11.0    & 74.2 (73.6) $\pm$ 9.5    & 68.8  & 71.0  & 67.5  & 68.8  & 66.2  \\
      AlpacaEval  & GPT 4o              & 14.8 (3.8) $\pm$ 21.1   & 74.8 (74.3) $\pm$ 8.9    & 66.5  & 68.5  & 64.8  & 67.3  & 64.4  \\
      AlpacaEval  & GPT 4o-mini         & 10.0 (4.1) $\pm$ 17.3   & 66.3 (65.1) $\pm$ 7.7    & 62.3  & 67.7  & 63.7  & 63.8  & 62.9  \\
      AlpacaEval  & GPT 5.4-mini        & 22.1 (11.7) $\pm$ 28.0  & 67.5 (68.5) $\pm$ 7.0    & 66.7  & 68.5  & 62.9  & 68.1  & 63.3  \\
      AlpacaEval  & GPT 5.4-nano        & 7.4 (2.3) $\pm$ 17.5    & 73.7 (67.5) $\pm$ 15.2   & 65.6  & 65.4  & 58.1  & 63.3  & 50.8  \\
      \midrule
      Arena       & DeepSeek v3.1 Chat  & 26.4 (21.5) $\pm$ 23.0  & 70.0 (70.9) $\pm$ 6.3    & 65.2  & 67.9  & 67.0  & 66.1  & 63.7  \\
      Arena       & Gemini 2.5 Flash    & 16.7 (4.6) $\pm$ 24.2   & 65.1 (65.2) $\pm$ 7.3    & 66.1  & 67.7  & 65.2  & 69.5  & 64.6  \\
      Arena       & GPT 4o              & 12.9 (4.1) $\pm$ 19.0   & 68.0 (66.7) $\pm$ 7.8    & 61.8  & 68.0  & 63.7  & 66.2  & 63.6  \\
      Arena       & GPT 4o-mini         & 11.7 (0.9) $\pm$ 23.5   & 78.4 (74.4) $\pm$ 17.8   & 62.4  & 69.3  & 62.9  & 65.1  & 61.9  \\
      Arena       & GPT 5.4-mini        & 14.0 (1.3) $\pm$ 29.6   & 77.3 (73.9) $\pm$ 14.0   & 63.6  & 68.0  & 64.0  & 65.4  & 63.6  \\
      Arena       & GPT 5.4-nano        & 16.9 (6.5) $\pm$ 20.6   & 66.2 (66.0) $\pm$ 7.4    & 61.9  & 60.5  & 54.9  & 61.3  & 49.4  \\
      \midrule
      PRISM       & DeepSeek v3.1 Chat  & 12.8 (2.3) $\pm$ 27.1   & 73.5 (66.0) $\pm$ 16.3   & 70.9  & 74.4  & 70.6  & 80.4  & 67.7  \\
      PRISM       & Gemini 2.5 Flash    & 14.9 (7.6) $\pm$ 21.8   & 65.2 (63.8) $\pm$ 8.1    & 70.8  & 75.5  & 73.0  & 77.6  & 72.9  \\
      PRISM       & GPT 4o              & 13.3 (6.8) $\pm$ 19.9   & 65.9 (64.9) $\pm$ 7.9    & 72.1  & 73.8  & 69.1  & 72.1  & 68.6  \\
      PRISM       & GPT 4o-mini         & 15.5 (6.7) $\pm$ 18.3   & 68.7 (67.4) $\pm$ 9.2    & 68.1  & 71.9  & 66.5  & 74.0  & 65.9  \\
      PRISM       & GPT 5.4-mini        & 8.4 (1.2) $\pm$ 19.2    & 77.8 (71.7) $\pm$ 18.2   & 66.4  & 73.1  & 67.5  & 71.9  & 67.5  \\
      PRISM       & GPT 5.4-nano        & 19.3 (7.8) $\pm$ 24.2   & 66.0 (65.5) $\pm$ 6.4    & 69.4  & 69.9  & 64.9  & 70.3  & 60.5  \\
      \bottomrule
    \end{tabular}
  \end{subtable}\hfill
  \begin{subtable}[t]{\textwidth}
    \centering
    \caption{Test}
    \begin{tabular}{llllccccc}
      \toprule
      Dataset     & Model               & Coverage              & Accuracy              & LLM   & MV    & PV    & GBM   & DT \\
      \midrule
      AlpacaEval  & DeepSeek v3.1 Chat  & 6.9 (0.0) $\pm$ 15.3    & 59.4 (59.7) $\pm$ 22.3   & 68.0  & 61.7  & 61.7  & 60.9  & 54.7  \\
      AlpacaEval  & GPT 4o              & 11.8 (7.4) $\pm$ 11.6   & 64.7 (63.7) $\pm$ 14.9   & 64.1  & 64.1  & 62.5  & 64.1  & 62.5  \\
      AlpacaEval  & GPT 4o-mini         & 14.7 (4.3) $\pm$ 21.6   & 58.3 (59.3) $\pm$ 25.6   & 63.3  & 59.4  & 53.9  & 56.2  & 46.9  \\
      AlpacaEval  & GPT 5.4-mini        & 8.8 (5.9) $\pm$ 9.9     & 53.0 (50.0) $\pm$ 26.5   & 66.4  & 58.6  & 62.5  & 59.4  & 62.5  \\
      AlpacaEval  & GPT 5.4-nano        & 21.2 (7.4) $\pm$ 27.8   & 55.2 (50.0) $\pm$ 17.6   & 60.2  & 58.6  & 57.8  & 57.8  & 50.8  \\
      AlpacaEval  & Gemini 2.5 Flash    & 8.1 (1.6) $\pm$ 19.7    & 67.1 (70.2) $\pm$ 25.6   & 67.2  & 65.6  & 64.1  & 64.8  & 64.1  \\
      \midrule
      Arena       & DeepSeek v3.1 Chat  & 24.9 (14.5) $\pm$ 24.1  & 58.5 (57.9) $\pm$ 6.2    & 63.0  & 59.8  & 61.8  & 58.5  & 60.3  \\
      Arena       & GPT 4o              & 19.0 (6.1) $\pm$ 26.2   & 57.3 (57.9) $\pm$ 16.4   & 58.8  & 58.0  & 59.3  & 60.8  & 59.0  \\
      Arena       & GPT 4o-mini         & 11.6 (3.1) $\pm$ 14.6   & 57.2 (58.3) $\pm$ 11.6   & 58.5  & 56.0  & 58.0  & 59.5  & 57.3  \\
      Arena       & GPT 5.4-mini        & 13.4 (1.0) $\pm$ 25.2   & 52.0 (52.8) $\pm$ 10.0   & 62.3  & 60.0  & 58.0  & 58.3  & 56.5  \\
      Arena       & GPT 5.4-nano        & 15.0 (1.0) $\pm$ 30.2   & 51.4 (51.8) $\pm$ 5.3    & 58.3  & 52.7  & 55.0  & 51.5  & 49.7  \\
      Arena       & Gemini 2.5 Flash    & 14.7 (6.2) $\pm$ 16.7   & 59.0 (59.0) $\pm$ 10.9   & 61.8  & 56.8  & 55.5  & 59.8  & 55.2  \\
      \midrule
      PRISM       & DeepSeek v3.1 Chat  & 15.6 (7.8) $\pm$ 23.7   & 62.5 (63.3) $\pm$ 17.8   & 65.7  & 65.2  & 61.0  & 64.5  & 64.0  \\
      PRISM       & GPT 4o              & 15.7 (7.5) $\pm$ 22.4   & 67.2 (64.8) $\pm$ 14.4   & 68.8  & 66.0  & 66.0  & 69.0  & 65.5  \\
      PRISM       & GPT 4o-mini         & 13.8 (7.5) $\pm$ 19.0   & 63.6 (64.0) $\pm$ 20.7   & 66.0  & 60.8  & 57.8  & 61.5  & 52.2  \\
      PRISM       & GPT 5.4-mini        & 27.6 (17.9) $\pm$ 23.6  & 66.2 (62.8) $\pm$ 23.2   & 70.3  & 64.7  & 64.2  & 67.0  & 64.2  \\
      PRISM       & GPT 5.4-nano        & 32.4 (17.1) $\pm$ 31.2  & 55.8 (51.5) $\pm$ 17.4   & 69.3  & 59.0  & 55.7  & 61.5  & 51.2  \\
      PRISM       & Gemini 2.5 Flash    & 22.7 (9.5) $\pm$ 26.9   & 64.9 (68.6) $\pm$ 11.9   & 69.3  & 65.2  & 67.5  & 67.5  & 66.5  \\
      \bottomrule
    \end{tabular}
  \end{subtable}
\end{table}

\begin{table}[H]
  \centering
  \footnotesize
  \setlength{\tabcolsep}{4pt}
  \renewcommand{\arraystretch}{1.10}
  \caption{Full reconstruction accuracy (\%) for ICAI+ (improved) principles across all datasets and models, along with coverage and accuracy (\%) of generated principles, presented as mean (median) $\pm$ standard deviation, as in Table~\ref{tab:full_results_icai}.}
  \label{tab:full_results_icaiplus}
  \begin{subtable}[t]{\textwidth}
    \centering
    \caption{Train}
    \begin{tabular}{llllccccc}
      \toprule
      Dataset     & Model               & Coverage              & Accuracy              & LLM   & MV    & PV    & GBM   & DT \\
      \midrule
      AlpacaEval  & DeepSeek v3.1 Chat  & 9.2 (6.2) $\pm$ 13.2    & 63.3 (61.9) $\pm$ 7.3    & 68.5  & 72.5  & 66.3  & 70.8  & 66.3  \\
      AlpacaEval  & Gemini 2.5 Flash    & 59.8 (63.9) $\pm$ 24.1  & 67.2 (67.4) $\pm$ 2.7    & 67.7  & 69.6  & 70.6  & 69.4  & 68.8  \\
      AlpacaEval  & GPT 4o              & 43.4 (34.2) $\pm$ 28.4  & 71.0 (70.6) $\pm$ 3.2    & 69.0  & 71.2  & 71.3  & 70.4  & 71.0  \\
      AlpacaEval  & GPT 4o-mini         & 71.7 (83.8) $\pm$ 25.6  & 65.5 (65.6) $\pm$ 2.4    & 65.0  & 69.2  & 67.9  & 71.5  & 67.7  \\
      AlpacaEval  & GPT 5.4-mini        & 59.4 (73.5) $\pm$ 31.2  & 68.4 (68.0) $\pm$ 3.0    & 67.3  & 71.5  & 67.7  & 70.6  & 67.9  \\
      AlpacaEval  & GPT 5.4-nano        & 39.5 (20.2) $\pm$ 40.4  & 65.4 (65.0) $\pm$ 5.8    & 67.3  & 66.3  & 61.9  & 61.3  & 59.8  \\
      \midrule
      Arena       & DeepSeek v3.1 Chat  & 55.8 (59.2) $\pm$ 23.9  & 71.9 (73.2) $\pm$ 3.5    & 65.2  & 68.1  & 65.6  & 68.6  & 65.4  \\
      Arena       & Gemini 2.5 Flash    & 62.2 (62.6) $\pm$ 28.9  & 69.4 (70.0) $\pm$ 3.2    & 65.6  & 69.0  & 67.5  & 68.2  & 68.0  \\
      Arena       & GPT 4o              & 75.6 (79.6) $\pm$ 22.0  & 70.2 (70.9) $\pm$ 2.2    & 62.9  & 64.6  & 64.7  & 68.0  & 64.4  \\
      Arena       & GPT 4o-mini         & 23.2 (20.6) $\pm$ 18.9  & 60.9 (59.9) $\pm$ 3.9    & 63.0  & 66.0  & 63.9  & 66.4  & 64.2  \\
      Arena       & GPT 5.4-mini        & 54.5 (48.4) $\pm$ 36.3  & 68.2 (68.6) $\pm$ 3.0    & 64.1  & 67.0  & 65.6  & 70.5  & 65.4  \\
      Arena       & GPT 5.4-nano        & 59.4 (54.3) $\pm$ 21.8  & 69.1 (69.2) $\pm$ 2.2    & 61.1  & 65.5  & 61.9  & 67.4  & 61.8  \\
      \midrule
      PRISM       & DeepSeek v3.1 Chat  & 35.0 (19.6) $\pm$ 32.3  & 57.9 (56.4) $\pm$ 5.3    & 72.5  & 74.6  & 71.1  & 74.8  & 70.4  \\
      PRISM       & Gemini 2.5 Flash    & 63.4 (75.7) $\pm$ 30.9  & 64.1 (63.6) $\pm$ 3.5    & 71.6  & 75.5  & 72.6  & 74.3  & 72.9  \\
      PRISM       & GPT 4o              & 58.6 (60.8) $\pm$ 31.2  & 62.4 (63.0) $\pm$ 3.2    & 71.9  & 74.6  & 73.1  & 73.3  & 73.3  \\
      PRISM       & GPT 4o-mini         & 67.0 (77.8) $\pm$ 27.6  & 63.4 (64.5) $\pm$ 2.7    & 72.3  & 74.4  & 71.4  & 73.5  & 71.3  \\
      PRISM       & GPT 5.4-mini        & 26.8 (15.4) $\pm$ 28.3  & 60.3 (59.9) $\pm$ 5.1    & 68.0  & 76.0  & 71.9  & 76.1  & 71.9  \\
      PRISM       & GPT 5.4-nano        & 58.9 (61.2) $\pm$ 26.7  & 66.8 (67.1) $\pm$ 2.9    & 67.0  & 68.2  & 61.6  & 71.8  & 58.6  \\
      \bottomrule
    \end{tabular}
  \end{subtable}\hfill
  \begin{subtable}[t]{\textwidth}
    \centering
    \caption{Test}
    \begin{tabular}{llllccccc}
      \toprule
      Dataset     & Model               & Coverage              & Accuracy              & LLM   & MV    & PV    & GBM   & DT \\
      \midrule
      AlpacaEval  & DeepSeek v3.1 Chat  & 9.2 (5.1) $\pm$ 13.3    & 55.4 (56.9) $\pm$ 25.3  & 72.7  & 75.0  & 68.8  & 73.4  & 63.3 \\
      AlpacaEval  & GPT 4o              & 59.2 (66.0) $\pm$ 23.1  & 68.3 (69.6) $\pm$ 5.2   & 68.0  & 66.4  & 67.2  & 67.2  & 67.2 \\
      AlpacaEval  & GPT 4o-mini         & 41.3 (31.6) $\pm$ 27.6  & 70.7 (68.8) $\pm$ 7.7   & 64.1  & 67.2  & 56.2  & 62.5  & 55.5 \\
      AlpacaEval  & GPT 5.4-mini        & 72.0 (81.2) $\pm$ 27.2  & 63.7 (63.3) $\pm$ 5.5   & 67.2  & 69.5  & 60.9  & 65.6  & 60.9 \\
      AlpacaEval  & GPT 5.4-nano        & 59.3 (75.0) $\pm$ 32.3  & 68.3 (68.7) $\pm$ 7.9   & 66.4  & 57.0  & 54.7  & 60.9  & 59.4 \\
      AlpacaEval  & Gemini 2.5 Flash    & 37.9 (16.8) $\pm$ 41.3  & 64.8 (63.7) $\pm$ 13.1  & 69.5  & 68.0  & 61.7  & 66.4  & 61.7 \\
      \midrule
      Arena       & DeepSeek v3.1 Chat  & 55.1 (56.1) $\pm$ 24.1  & 67.9 (67.9) $\pm$ 5.2   & 62.7  & 61.8  & 59.5  & 63.0  & 58.8 \\
      Arena       & GPT 4o              & 61.0 (63.4) $\pm$ 29.0  & 66.4 (66.1) $\pm$ 2.9   & 59.8  & 59.5  & 61.0  & 61.8  & 60.8 \\
      Arena       & GPT 4o-mini         & 74.8 (79.6) $\pm$ 22.8  & 65.7 (66.6) $\pm$ 2.6   & 60.0  & 59.5  & 58.0  & 57.8  & 56.5 \\
      Arena       & GPT 5.4-mini        & 23.2 (19.0) $\pm$ 19.4  & 54.9 (55.4) $\pm$ 7.7   & 61.3  & 62.5  & 60.8  & 60.3  & 60.5 \\
      Arena       & GPT 5.4-nano        & 53.4 (45.8) $\pm$ 36.3  & 63.9 (65.6) $\pm$ 8.4   & 56.2  & 56.8  & 58.8  & 59.8  & 59.0 \\
      Arena       & Gemini 2.5 Flash    & 58.3 (55.5) $\pm$ 22.3  & 65.3 (65.0) $\pm$ 3.2   & 63.7  & 60.8  & 61.5  & 61.0  & 62.0 \\
      \midrule
      PRISM       & DeepSeek v3.1 Chat  & 36.1 (20.6) $\pm$ 31.9  & 54.8 (57.2) $\pm$ 8.2   & 68.2  & 68.0  & 64.2  & 69.8  & 62.3 \\
      PRISM       & GPT 4o              & 60.5 (71.6) $\pm$ 30.7  & 59.8 (59.5) $\pm$ 3.6   & 69.8  & 66.7  & 67.2  & 68.2  & 67.2 \\
      PRISM       & GPT 4o-mini         & 55.9 (54.2) $\pm$ 32.7  & 59.2 (58.6) $\pm$ 4.3   & 66.5  & 68.8  & 67.5  & 67.7  & 67.2 \\
      PRISM       & GPT 5.4-mini        & 65.2 (78.6) $\pm$ 27.9  & 60.2 (60.0) $\pm$ 3.7   & 69.5  & 68.8  & 69.3  & 68.2  & 69.3 \\
      PRISM       & GPT 5.4-nano        & 25.4 (11.9) $\pm$ 28.3  & 53.4 (54.3) $\pm$ 8.0   & 65.5  & 58.5  & 61.3  & 61.3  & 61.0 \\
      PRISM       & Gemini 2.5 Flash    & 58.1 (62.1) $\pm$ 26.5  & 60.9 (61.4) $\pm$ 3.1   & 67.5  & 68.5  & 67.0  & 68.2  & 69.0 \\
      \bottomrule
    \end{tabular}
  \end{subtable}
\end{table}

%% file: appendices/prompt_templates.tex
\section{Prompt templates}\label{app:prompts}

\subsection{Principle discovery prompt (ICAI / ICAI+)}
\begin{verbatim}
System: You are an expert preference analyst specializing in
reverse-engineering human judgment. Your task is to identify the
precise, distinguishing features that make one text superior to
another and articulate these distinctions as diverse, actionable
principles.

User: Analyze the following pair of responses.

Response A (Preferred):
{preferred_sample}

Response B (Rejected):
{rejected_sample}

Generate {num_principles} DIVERSE and NUANCED principles explaining
why Response A was preferred over Response B.

CRITICAL INSTRUCTIONS:
1. Contrastive Analysis: Identify the specific, concrete differences
   (the "delta") between A and B. The principles MUST be grounded in
   this contrast, not just general writing advice.
2. Identify Salience: Focus only on the most critical factors that
   explain the preference.
3. Abstraction: Generalize these specific differences into principles
   applicable to other situations.
4. Diversity: Ensure each principle captures a conceptually distinct
   dimension of quality. Do not provide variations of the same concept.

Format Requirements:
- Max 20 words per principle.
- Start every principle with 'Select the response that...'.

Reply as JSON: {"principles": ["<PRINCIPLE 1>", "<PRINCIPLE 2>",...]}.
\end{verbatim}

\subsection{Principle refinement prompt (ICAI+ only)}
\begin{verbatim}
System: You are an expert system for refining Principles that identify
key preference patterns in text. Your task is to specialize a Principle
to increase its accuracy by correcting misclassifications.

User: The Current Principle below is inaccurate and frequently disagrees
with ground truth annotations. We need to make it more precise or add
necessary context/exceptions.

### Current Principle
{current_principle}

### Inputs
The Current Principle makes the WRONG judgment on these comparisons.
The Revised Principle must aim to correct these:
{failed_comparisons}

The Current Principle makes the CORRECT judgment on these comparisons.
The Revised Principle MUST strictly maintain these correct judgments:
{succeeded_comparisons}

### Task and Guidelines
Propose a revised, more accurate Principle. Follow these steps:
1. Analyze Contrast: Examine differences between Failure and Anchor
   Comparisons to identify missing nuance or context.
2. Identify Failure Patterns: Determine specific conditions under
   which the principle fails.
3. Specialize: Introduce conditional logic, specify exceptions, or
   clarify ambiguity with more precise language.
4. Verify: Ensure the Revised Principle corrects failures while
   preserving correct judgments.

Revised Principle (MAXIMUM {word_limit} WORDS):
\end{verbatim}

\subsection{Principle annotation prompt (votes: A/B/N/A)}
\begin{verbatim}
System: You apply general quality criteria to compare texts. Be
decisive in your judgments.

User:
Sample A:
{sample_a}

Sample B:
{sample_b}

For each rule below, determine which sample better follows it:
{principles}

Answer in JSON format, e.g. {0: "A", 1: "B", 2: "A",...}.
- Put "A" if A follows the rule better
- Put "B" if B follows the rule better
- Only use "None" if the rule truly cannot apply to either text
- When uncertain, choose the one that fits better overall
\end{verbatim}

\subsection{LLM-as-a-judge executor prompt}
\begin{verbatim}
System: You are an impartial judge evaluating two responses according
to a set of principles.

User:
Prompt: {prompt}

Response A:
{response_a}

Response B:
{response_b}

Constitution (principles to apply):
{constitution}

Based on the principles above, which response is better overall?
Consider how well each response adheres to the principles.
Reply with only "A" or "B".
\end{verbatim}

\paragraph{Decoding parameters.} All LLM calls used temperature $= 0$, and default values for other parameters. Maximum output tokens were set to 1024 for principle discovery and 256 for annotation and execution.